\documentclass[twoside]{article}

\usepackage[accepted]{aistats2021}
\usepackage[round]{natbib}

\usepackage[utf8]{inputenc} 
\usepackage[T1]{fontenc}    
\usepackage{hyperref}       
\usepackage{url}            
\usepackage{booktabs}       
\usepackage{amsfonts}       
\usepackage{nicefrac}       
\usepackage{microtype}      

\usepackage[table]{xcolor}
\usepackage{amssymb}
\usepackage{amsthm}
\usepackage{amsmath}
\usepackage{algorithm}
\usepackage{algpseudocode}
\usepackage{hyperref}
\usepackage{graphicx}
\usepackage[font={small}]{caption}
\usepackage{thm-restate}
\usepackage[normalem]{ulem}
\usepackage{xspace}
\usepackage{arydshln}
\usepackage{wrapfig}
\usepackage{enumitem}
\usepackage{framed}
\usepackage{dsfont}
\usepackage{balance}
\usepackage{subcaption}
\usepackage{arydshln}

\usepackage{multirow}

\allowdisplaybreaks

\algnotext{EndFor}
\algnotext{EndIf}
\algnotext{EndWhile}
\algnotext{EndProcedure}
\algnotext{Procedure}

\newcommand{\A}{\mathbb{A}}
\newcommand{\B}{\mathcal{B}}
\newcommand{\E}{\mathbb{E}}
\newcommand{\bw}{\mathbf{w}}
\newcommand{\bzero}{\mathbf{0}}
\newcommand{\bone}{\mathbf{1}}

\newcommand{\bz}{\mathbf{z}}
\newcommand{\bq}{\mathbf{q}}
\newcommand{\bhw}{\hat{\mathbf{w}}}

\newcommand{\bv}{\mathbf{v}}
\newcommand{\bu}{\mathbf{u}}
\newcommand{\bg}{\mathbf{g}}
\newcommand{\bx}{\mathbf{x}}

\newcommand{\bX}{\mathbf{X}}
\newcommand{\be}{\mathbf{e}}
\newcommand{\bs}{\mathbf{s}}
\newcommand{\cD}{\mathcal{D}}
\newcommand{\cQ}{\mathcal{Q}}

\newcommand{\tcQ}{\tilde{\cQ}}
\newcommand{\cN}{\mathcal{N}}
\newcommand{\cB}{\mathcal{B}}
\newcommand{\cO}{\mathcal{O}}
\newcommand{\cM}{\mathcal{M}}
\newcommand{\bB}{\mathbb{B}}
\newcommand{\bP}{\mathbb{P}}
\newcommand{\N}{\mathbb{N}}
\newcommand{\R}{\mathbb{R}}
\newcommand{\cP}{\mathcal{P}}
\newcommand{\tcP}{\tilde{\cP}}
\newcommand{\tcB}{\tilde{\cB}}

\newcommand{\cX}{\mathcal{X}}
\newcommand{\cU}{\mathcal{U}}
\newcommand{\cC}{\mathcal{C}}
\newcommand{\cR}{\mathcal{R}}
\newcommand{\eps}{\epsilon}

\newcommand{\gammagap}{\gamma_{\text{gap}}}

\DeclareMathOperator{\sgn}{sgn}
\DeclareMathOperator{\err}{err}

\DeclareMathOperator{\poly}{poly}
\DeclareMathOperator{\supp}{supp}
\DeclareMathOperator{\Beta}{Beta}

\DeclareMathOperator{\Lap}{Lap}
\DeclareMathOperator{\argmax}{argmax}

\newtheorem{theorem}{Theorem}
\newtheorem{definition}[theorem]{Definition}
\newtheorem{lemma}[theorem]{Lemma}
\newtheorem{claim}[theorem]{Claim}

\newtheorem{observation}[theorem]{Observation}

\newcommand{\1}{\mathds{1}}

\begin{document}

%

%

\twocolumn[

\aistatstitle{Robust and Private Learning of Halfspaces}

\aistatsauthor{Badih Ghazi \And
  Ravi Kumar \And
  Pasin Manurangsi \And
  Thao Nguyen}

\aistatsaddress{Google Research \\ Mountain View, CA \\
\texttt{ \{badihghazi, ravi.k53\}@gmail.com, \{pasin, thaotn\}@google.com }} 
]

\begin{abstract}
In this work, we study the trade-off between differential privacy and adversarial robustness under $L_2$-perturbations in the context of learning halfspaces. We prove nearly tight bounds on the sample complexity of robust private learning of halfspaces for a large regime of parameters.  A highlight of our results is that  robust \emph{and} private learning is harder than robust \emph{or} private learning alone.  We complement our theoretical analysis with experimental results on the MNIST and USPS datasets, for a learning algorithm that is both differentially private and adversarially robust.
\end{abstract}

\section{Introduction}

In this work, we study the interplay between two topics at the core of AI ethics and safety: privacy and robustness.

As modern machine learning models are trained on potentially sensitive data, there has been a tremendous interest in privacy-preserving training methods. \emph{Differential privacy (DP)}~\citep{DworkMNS06,DworkKMMN06} has emerged as the gold standard for rigorously tracking the privacy leakage of algorithms in general~\citep[see, e.g.,][and the references therein]{dwork2014algorithmic,vadhan2017complexity}, and machine learning models in particular~\citep[e.g.,][]{abadi2016deep}, resulting in several practical deployments in recent years~\citep[e.g.,][]{erlingsson2014rappor,CNET2014Google, greenberg2016apple, dp2017learning, ding2017collecting, abowd2018us}.

Another vulnerability of machine learning models that has also been widely studied recently is with respect to adversarial manipulations of their inputs at test time, with the intention of causing classification errors~\citep[e.g.,][]{dalvi2004adversarial,biggio2013evasion,szegedy2013intriguing,goodfellow2014explaining,papernot2016transferability}. Numerous methods have been proposed with the goal of training models that are robust to such adversarial attacks~\citep[e.g.,][]{madry2017towards,gowal2018effectiveness, gowal2019alternative, schott2018towards}, which in turn has led to new attacks being devised in order to fool these models~\citep{athalye2018obfuscated, carlini2018audio,sharma2017attacking}. See~\citep{kolter2018adversarial} for a recent tutorial on this topic.

Some recent work has suggested incorporating mechanisms from DP into neural network training to enhance 
adversarial robustness~\citep{lecuyer2019certified, phan2019scalable, phan2019heterogeneous}. Given this existing interplay between DP and robustness, we seek to answer the following natural question: \\

\begin{minipage}{0.5\textwidth}
\vspace{-0.5cm}
\begin{center}
\textsl{Is achieving privacy and adversarial robustness \\ harder than achieving either criterion alone?}
\end{center}
\end{minipage}

Recent empirical work has provided mixed response to the question \citep{song2019privacy, song2019membership, hayes2020provable}, reporting the success rate of membership inference attacks as a heuristic measure of privacy. Instead, using theoretical analysis and the strict guarantees offered by DP, 
we formally investigate this question in the classic setting of halfspace learning, and arrive at a near-complete picture.

\paragraph{Background.}

In order to present our results, we start by recalling some notions from robust learning in the PAC model.
Let $\mathcal{C} \subseteq \{0, 1\}^{\cX}$ be a (Boolean) hypothesis class on an instance space $\mathcal{X} \subseteq \R^d$. 
A \emph{perturbation} is defined by a function $\bP: \mathcal{X} \to 2^{\mathcal{X}}$, where $\bP(x) \subseteq \cX$ denotes the set of allowable perturbed instances starting from an instance $x$. 
The \emph{robust risk} of a hypothesis $h$ with respect to a distribution 
$\mathcal{D}$ on $\cX \times \{ \pm 1\}$ and perturbation $\bP$ is defined as 
$\mathcal{R}_{\bP}(h, \mathcal{D}) = 
\Pr_{(\bx, y) \sim \mathcal{D}}[\exists z \in \bP(\bx), h(\mathbf{z}) \neq y].$
A distribution $\cD$ is said to be \emph{realizable} (with respect to $\cC$ and $\bP$) iff there exists $h^* \in \cC$ such that $\cR_{\bP}(h^*, \cD) = 0$.
In the adversarially robust PAC learning problem, the learner is given i.i.d. samples from a realizable distribution $\mathcal{D}$ on $\mathcal{X} \times \{ \pm 1\}$, and the goal is to output a hypothesis $h: \cX \to \{\pm 1\}$ such that with probability $1 - \xi$ it holds that $\mathcal{R}_{\bP}(h, \mathcal{D}) \leq \alpha$. We refer to $\xi$ as the \emph{failure probability} and $\alpha$ as the \emph{accuracy parameter}. A learner is said to be \emph{proper} if the output hypothesis $h$ belongs to $\cC$; otherwise, it is said to be \emph{improper}.

We focus our study on the concept class of {\em halfspaces}, i.e., $\cC_{\text{halfspaces}} := \{h_{\bw} \mid \bw \in \R^d\}$ where $h_{\bw}(\bx) = \sgn(\left<\bw, \bx\right>)$, in this model with respect to $L_2$ perturbations, i.e., $\bP_\gamma(\bx) := \{ \mathbf{z} \in \mathcal{X}: \|\bz-\bx\|_2 \leq \gamma \}$ for \emph{margin} parameter $\gamma > 0$. We assume throughout that the domain of our functions is bounded in the $d$-dimensional Euclidean unit ball $\bB^d := \{\bx \in \R^d \mid \|\bx\|_2 \leq 1\}$. We also write $\cR_{\gamma}$ as a shorthand for $\cR_{\bP_{\gamma}}$.
%
%
An algorithm is said to be a \emph{$(\gamma, \gamma')$-robust learner} if, for any realizable distribution $\cD$ with respect to $\cC_{\text{halfspaces}}$ and $\bP_{\gamma}$, using a certain number of samples, it outputs a hypothesis $h$ such that w.p. $1 - \xi$, we have $\cR_{\gamma'}(h, \mathcal{D}) \leq \alpha$, where $\alpha, \xi > 0$ are sufficiently smaller than some positive constant.  We are especially interested in the case\footnote{It is necessary to have $\gamma' < \gamma$; when $\gamma = \gamma'$, proper $(\gamma, \gamma')$-robust learning is as hard as general proper learning of halfspace~\citep[see, e.g.,][]{DKM20}, which is impossible with any finite number of samples under DP~\citep{BunNSV15}.} where $\gamma'$ is close to $\gamma$; for simplicity, we use $\gamma' = 0.9\gamma$ as a representative setting throughout.  In robust learning, the main quantities of interest are the \emph{sample complexity}, i.e., the minimum number of samples needed to learn, and the \emph{running time} of the learning algorithm.

We use the standard terminology of DP. Recall that two datasets $\bX$ and $\bX'$ are \emph{neighbors} if $\bX'$ results from adding or removing a single data point from $\bX$.%

\begin{definition}[Differential Privacy (DP)~\citep{DworkMNS06,DworkKMMN06}]
Let $\epsilon, \delta \in \R_{\geq 0}$. A randomized algorithm $\A$ taking as input a dataset is said to be \emph{$(\epsilon, \delta)$-differentially private} (denoted by $(\epsilon, \delta)$-DP) if for any two neighboring datasets $\bX$ and $\bX'$, and for any subset $S$ of outputs of $\A$, it is the case that $\Pr[\A(\bX) \in S] \le e^{\epsilon} \cdot \Pr[\A(\bX') \in S] + \delta$. If $\delta = 0$, $\A$ is said to be \emph{$\epsilon$-differentially private} (denoted by $\epsilon$-DP).
\end{definition}

As usual, $\eps$ should be thought of as a small constant, whereas $\delta$ should be negligible in the dataset size. We refer to the case where $\delta = 0$ as \emph{pure-DP}, and the case where $\delta > 0$ as \emph{approximate-DP}.


\paragraph{Our Results.}

\definecolor{ashgrey}{rgb}{0.7, 0.75, 0.71}

\begin{table*}[t]
    \centering
    \footnotesize
\begin{tabular}{|c|c|c|c|c|c|}
    \cline{3-6}
      \multicolumn{2}{c|}{} & \multicolumn{2}{c|}{Robust} & \multicolumn{2}{c|}{Non-robust}\\
    \cline{3-6}
          \multicolumn{1}{c}{} &
          \multicolumn{1}{c|}{Bounds} & Proper & Improper & Proper & Improper\\ 
         \hline
         Non-private & Tight & \multicolumn{4}{c|}{$\Theta(1/\gamma^2)$} \\
         
         \cline{2-6}
         \hline
         \multirow{2}*{Pure-DP} & Upper & \multicolumn{2}{c|}{\cellcolor{white} $O(d)$~(Theorem~\ref{thm:exp-mech})} & \multicolumn{2}{c|}{$\tilde{O}(1/\gamma^2)$}\\ 
         \cdashline{2-6}
         & Lower & \multicolumn{2}{c|}{
         $\Omega(d)$~(Theorem~\ref{thm:lower-bound})} & $\Omega(1/\gamma^2)$ & $\Omega(1/\gamma^2)$~(Theorem~\ref{thm:lower-bound-non-robust})\\
      \hline
      \multirow{2}*{Approximate-DP} & Upper & \multicolumn{2}{c|}{$\tilde{O}(\sqrt{d}/\gamma)$~$^\dag$} & \multicolumn{2}{c|}{$\tilde{O}(1/\gamma^2)$} \\ 
      \cdashline{2-6}
      & Lower &  $\Omega(\sqrt{d}/\gamma)$~(Theorem~\ref{thm:lower-bound-apx-proper}) &  $\Omega(1/\gamma^2)$~(Theorem~\ref{thm:lower-bound-non-robust}) & \multicolumn{2}{c|}{
      $\Omega(1/\gamma^2)$~(Theorem~\ref{thm:lower-bound-non-robust})}\\
    \hline
    \end{tabular}
     \caption{Trade-offs between privacy and  robustness.  The \emph{robust} column corresponds to $(\gamma, 0.9\gamma)$-robust learners, whereas the \emph{non-robust} column corresponds to $(\gamma, 0)$-robust learners. For simplicity of presentation, we assume $\alpha, \eps, \xi \in (0, 1)$ are sufficiently small constants, $d \geq 1/\gamma^2$, and for approximate-DP lower bounds, that $\delta = o(1/n)$. While  approximate-DP upper bounds (marked with$^\dag$) can already be derived from previous work, we give a faster algorithm (Theorem~\ref{thm:perceptron-apx-DP-simplified}). For DP, known results are from~\citep{NguyenUZ20}; for the non-private case, the results follow, e.g., from~\citep{bartlett2002rademacher,koltchinskii2002empirical}.}
    \label{table:summary}
  \end{table*}

We assume that $\eps \le O(1)$ unless otherwise stated, and that $\gamma, \alpha, \xi > 0$ are sufficiently smaller than some positive constant. We will not state these assumptions explicitly here for simplicity; interested readers may refer to (the first lines of) the proofs for the exact upper bounds that are imposed. 

We first prove that robust learning with pure-DP requires $\Omega(d)$ samples.  
\begin{restatable}{theorem}{lowerbound} \label{thm:lower-bound}
Any $\eps$-DP $(\gamma, 0.9\gamma)$-robust (possibly improper) learner has sample complexity $\Omega(d / \eps)$.  
\end{restatable}
In the private but non-robust setting (i.e., for an $\eps$-DP $(\gamma, 0)$-robust learner\footnote{This means that the output hypothesis only needs to have a small classification error, but may have a large robust risk.}), \cite{NguyenUZ20} showed that $O(1/\gamma^2)$ samples suffice. Together with earlier known results showing that $(\gamma, 0.9\gamma)$-robust learning (without privacy) only requires $O(1/\gamma^2)$ samples~\citep[e.g.,][]{bartlett2002rademacher,koltchinskii2002empirical}, our result gives a separation between robust private learning and private learning alone or robust learning alone, whenever $d \gg 1/\gamma^2$.

For the case of approximate-DP, we establish a lower bound of $\Omega(\min\{\sqrt{d}/\gamma, d\})$, which holds only against \emph{proper} learners. As for Theorem~\ref{thm:lower-bound}  in the context of pure-DP, this result implies a similar separation in the approximate-DP proper learning setting.

\begin{restatable}{theorem}{lowerboundapx}
\label{thm:lower-bound-apx-proper}
Let $\eps < 1$. Any $(\eps, o(1/n))$-DP $(\gamma, 0.9\gamma)$-robust proper learner has sample complexity $n = \Omega(\min\{\sqrt{d}/\gamma, d\})$.
\end{restatable}

Our proof technique can also be used to improve the lower bound for DP $(\gamma, 0)$-robust learning. Specifically, \cite{NguyenUZ20} show an $\Omega(1/\gamma^2)$ lower bound for \emph{proper} $(\gamma, 0)$-robust learners with \emph{pure}-DP.  We extend it to hold even for \emph{improper} $(\gamma, 0)$-robust learners with \emph{approximate}-DP:

\begin{restatable}{theorem}{lbnr}
\label{thm:lower-bound-non-robust}
For any $\eps > 0$, there exists $\delta > 0$ such that any $(\eps, \delta)$-DP $(\gamma, 0)$-robust (possibly improper) learner has sample complexity $\Omega\left(\frac{1}{\eps\gamma^2}\right)$. Moreover, this holds even when $d = O(1 / \gamma^2)$.
\end{restatable}

Finally, we provide algorithms with nearly matching upper bounds. For pure-DP, we prove the following, which matches the lower bound in Theorem~\ref{thm:lower-bound} to within a constant factor when $d \geq 1/\gamma^2$.


\begin{restatable}{theorem}{expmech}
\label{thm:exp-mech}
There is an $\eps$-DP $(\gamma, 0.9\gamma)$-robust learner with sample complexity%
\footnote{Here $O_\alpha(\cdot)$ hides a factor of $\poly(1/\alpha)$, and $\tilde{O}(\cdot)$ hides a factor of $\poly\log(1/(\alpha\gamma\delta))$.} $O_{\alpha}\left(\frac{1}{\epsilon} \max\{d, \frac{1}{\gamma^2}\}\right)$.
\end{restatable}

For approximate-DP, it is already possible\footnote{This can be achieved by running the DP ERM algorithm of~\citep{BassilyST14} with the hinge loss; see~\citep{NguyenUZ20} for the analysis.} to achieve a sample complexity%
\footnotemark[3] of $n = \tilde{O}_{\alpha}(\sqrt{d}/\gamma)$~\citep{NguyenUZ20,BassilyST14} but the running time is $\Omega(n^2 d)$.\footnote{Specifically,~\cite{NguyenUZ20} uses the DP empirical risk minimization algorithm of~\citep{BassilyST14} with the hinge loss; however, the latter requires $\Omega(n^2)$ iterations and each iteration requires $\Omega(d)$ time.} We give a faster algorithm with running time $O_{\alpha}(nd/\gamma)$.

\begin{restatable}{theorem}{perceptron}
\label{thm:perceptron-apx-DP-simplified}
There is an $(\eps, \delta)$-DP $(\gamma, 0.9\gamma)$-robust learner with sample complexity $n = \tilde{O}_\alpha\left(\frac{1}{\eps} \cdot \max\left\{\frac{\sqrt{d}}{\gamma}, \frac{1}{\gamma^2}\right\}\right)$ and running time $\tilde{O}_{\alpha}\left(n d / \gamma\right)$.
\end{restatable}

Our theoretical results and those from prior works are summarized in Table~\ref{table:summary}. Notice that the non-private robust setting and the non-robust private setting each requires only $O(1/\gamma^2)$ samples, whereas our results show that the private and robust setting requires either $\Omega(d)$ samples (for pure-DP) or $\Omega(\sqrt{d}/\gamma)$ samples (for approximate-DP). This separation positively answers the question central to our study. 


We complement our theoretical results by empirically evaluating our algorithm (from Theorem~\ref{thm:perceptron-apx-DP-simplified}) on the MNIST~\citep{lecun2010mnist} and USPS~\citep{hull1994database} datasets. Our results show that it is possible to achieve both robustness and privacy guarantees while maintaining reasonable performance. We further provide evidence that models trained via our algorithm are more resilient to adversarial noise compared to neural networks trained via DP-SGD~\citep{abadi2016deep}.

\paragraph{Organization.}

In the two following sections, we describe in detail the ideas behind each of our proofs. We then present our experimental results in Section~\ref{sec:exp}. Finally, we discuss additional related work and several open questions in Sections~\ref{sec:add-related-work} and~\ref{sec:disc_open_questions} respectively.
Due to space constraints, all missing proofs and additional experiments are deferred to the Supplementary Material.

\section{Sample Complexity Lower Bounds}
\label{sec:lb}
In this section, we explain the high-level ideas behind each of our sample complexity lower bounds.  Our pure-DP lower bound is based on a packing framework and our approximate-DP lower bounds are based on fingerprinting codes.

\subsection{Pure-DP Lower Bound (Theorem~\ref{thm:lower-bound})}
We use the \emph{packing framework}, a DP lower bound proof technique that originated in~\citep{HardtT10}. Roughly speaking, to apply this framework, we have to construct many input distributions for which the sets of valid outputs for each distribution are disjoint (hence the name ``packing''). In our context, this means that we would like to construct distributions $\cD^{(1)}, \dots, \cD^{(K)}$ such that the sets $G^{(i)}$ of hypotheses with small robust risk on $\cD^{(i)}$ are disjoint. Once we have done this, the packing framework immediately gives us a lower bound of $\Omega(\log K / \eps)$ on the sample complexity; below we describe a construction for $K = 2^{\Omega(d)}$ distributions, which yields the desired $\Omega(d / \eps)$ lower bound in Theorem~\ref{thm:lower-bound}.

Our construction proceeds by picking unit vectors $\bw^{(1)}, \dots, \bw^{(K)}$ that are nearly orthogonal, i.e., $\left|\left<\bw^{(i)}, \bw^{(j)}\right>\right| < 0.01$ for all $i \ne j$. It is not hard to see (and well-known) that such vectors exist for $K = 2^{\Omega(d)}$. We then let $\cD^{(i)}$ be the uniform distribution on $(1.01\gamma \cdot \bw^{(i)}, +1), (-1.01\gamma \cdot \bw^{(i)}, -1)$. 

Now, let $G^{(i)}$ denote the set of hypotheses $h$ for which $\cR_{0.9\gamma}(h, \cD^{(i)}) < 0.5$. Since our distribution $\cD^{(i)}$ is uniform on two elements, we must have that $\cR_{0.9\gamma}(h, \cD^{(i)}) = 0$ for all $h \in G^{(i)}$. To see that $G^{(i)}$ and $G^{(j)}$ are disjoint for any $i \ne j$, notice that, since $\left|\left<\bw^{(i)}, \bw^{(j)}\right>\right| < 0.01$, the point $1.01\gamma \cdot \bw^{(i)}$ is within distance $1.8\gamma$ from the point $-1.01\gamma \cdot \bw^{(j)}$. This means that any hypothesis $h$ cannot correctly classify  both $(1.01\gamma \cdot \bw^{(i)}, 1)$ and $(-1.01\gamma \cdot \bw^{(i)}, -1)$ with margin at least $0.9\gamma$, which implies that $G^{(i)} \cap G^{(j)} = \emptyset$. This completes our proof sketch.

We end by remarking that the previous work of~\cite{NguyenUZ20} also uses a packing argument; the main differences between our construction and theirs are in the choice of the distributions $\cD^{(i)}$ and the proof of disjointness of the $G^{(i)}$'s. Our construction of $\cD^{(i)}$ is in fact simpler, since our proof of disjointness can rely on the robustness guarantee. These differences are inherent, since our lower bound holds even against \emph{improper} learners, i.e., when the output hypothesis may not be a halfspace, whereas the lower bound of~\cite{NguyenUZ20} only holds against the particular case of \emph{proper} learners.


\subsection{Approximate-DP Lower Bound (Theorem~\ref{thm:lower-bound-apx-proper})} We reduce from a lower bound from the line of works~\citep{BunUV18,DworkSSUV15,SteinkeU16,SteinkeU17} inspired by \emph{fingerprinting codes}. Specifically, these works consider the so-called \emph{attribute mean} problem, where we are given a set of vectors drawn from some (hidden) distribution $\cD$ on $\{\pm 1\}^d$ and the goal is to compute the mean. It is known that getting an estimate to within 0.1 of the true mean in each coordinate requires $\Omega(\sqrt{d})$ samples. In fact,~\cite{SteinkeU17} show that even outputting a vector with a ``non-trivial'' dot product\footnote{Specifically, this holds when the output vector has $\ell_2$-norm at most $\sqrt{d}$ and the dot product is at least $\zeta d$ for any constant $\zeta > 0$.} with the mean already requires $\Omega(\sqrt{d})$ samples. This \emph{almost} implies our desired lower bound: the only remaining step is to turn $\cD$ to a distribution that is realizable with margin $\gamma^*$. We do this by conditioning $\cD$ on only points $\bx$ with a sufficiently large dot product with the true mean, and then adding both $(\bx, +1)$ and $(-\bx, -1)$ to our distribution. This reduction gives an $\Omega(\sqrt{d})$ lower bound on the sample complexity of $(\gamma^*, 0.9\gamma^*)$-robust proper learners, for some absolute constant margin parameter $\gamma^* > 0$. 

To get an improved bound for a smaller margin $\gamma$, we ``embed'' $\Omega(1/\gamma^2)$ hard instances above in each of $O(\gamma^2 d)$ dimensions. More specifically, let $T = \gamma^* / \gamma$; for each $i \in [T^2]$ we create a distribution $\cD^{(i)}$ that is the hard distribution from the previous paragraph in $d' := d / T$ dimensions embedded onto coordinates $d'(i - 1) + 1, \dots, d' i$. We then let the distribution $\cD'$ be the (uniform) mixture of $\cD^{(1)}, \dots, \cD^{(T^2)}$. Since each $\cD^{(i)}$ is realizable with margin $\gamma^*$ via some halfspace $\bw^{(i)}$, we may take $\bw^* := \frac{1}{T} \sum_{i \in [T^2]} \frac{\bw^{(i)}}{\|\bw^{(i)}\|}$ to realize the distribution $\cD'$ with margin $\frac{1}{T} \cdot \gamma^* = \gamma$ as desired. 

Now, to find any $\bw$ with small $\cR_{0.9\gamma}(h_{\bw}, \cD')$, we roughly have to solve (most of) the $T^2$ instances $\cD^{(i)}$'s. Recall that solving each of these instances requires $\Omega(\sqrt{d'})$ samples. Thus, the combined instance requires $\Omega(T^2 \cdot \sqrt{d'}) = \Omega(1/\gamma^2 \cdot \sqrt{\gamma^2 d}) = \Omega(\sqrt{d}/\gamma)$ samples.

\subsection{Non-Robust DP Learning Lower Bound (Theorem~\ref{thm:lower-bound-non-robust})} This lower bound once again uses the ``embedding'' technique described above. Here we start with a hard one-dimensional instance, which is simply the uniform distribution on $(x, -1), (x, +1)$ where $x$ is either $+1$ or $-1$. When $\delta > 0$ is sufficiently small (depending on $\eps$), it is simple to show that any $(\eps, \delta)$-DP $(1, 0)$-learner for this instance requires $\Omega(1/\eps)$ samples. Similar to the previous proof overview, we embed $1/\gamma^2$ such instances into $1/\gamma^2$ dimensions. Since the one-dimensional instance requires $\Omega(1/\eps)$ samples, the combined instance requires $\Omega(1/(\eps\gamma^2))$ samples, thereby yielding Theorem~\ref{thm:lower-bound-non-robust}.

\section{Sample-Efficient Algorithms}
\label{sec:alg}

In this section, we present our algorithms for robust and private learning.  Our pure-DP algorithm is based on an improved analysis of the exponential mechanism.  Our approximate-DP algorithm is based on a private, batched version of the perceptron algorithm.  

In the following discussions, we assume that $\bw^*$ is an (unknown) optimal halfspace with respect to the input distribution $\cD$ and $L_2$-perturbations with margin parameter $\gamma$, i.e., that $\bw^*$ satisfies $\cR_{\gamma}(h_{\bw^*}, \cD) = 0$. We may assume without loss of generality that $\|\bw^*\| = 1$.

\subsection{Pure-DP Algorithm (Theorem~\ref{thm:exp-mech})}
Theorem~\ref{thm:exp-mech} is shown via the exponential mechanism (EM)~\citep{McSherryT07}. Our guarantee is an improvement over the ``straightforward'' analysis of EM on a $(0.1\gamma)$-net of the unit sphere in $\R^d$, which gives an upper bound of $O_{\alpha}(d \log(1 / \gamma) / \eps)$~\citep{NguyenUZ20}. On the other hand, when $d \geq 1/\gamma^2$, our sample complexity is $O_{\alpha}(d / \epsilon)$. The intuition behind our improvement is that, if we take a random unit vector $\bw$ such that $\left<\bw, \bw^*\right> \geq 0.99$, then it already gives a small robust risk (in expectation) when $d \gg 1/\gamma^2$ because the component of $\bw$ orthogonal to $\bw^*$ is a random $(d - 1)$-dimensional vector of norm less than one, meaning that in expectation it only affects the margin by $O(1/\sqrt{d}) \ll 0.1\gamma$. Now, a random unit vector satisfies $\left<\bw, \bw^*\right> \geq 0.99$ with probability $2^{-O(d)}$, which (roughly speaking) means that EM should only require $O_{\alpha}(d/\eps)$ samples.

\subsection{Approximate-DP Algorithm (Theorem~\ref{thm:perceptron-apx-DP-simplified})} 

To prove Theorem~\ref{thm:perceptron-apx-DP-simplified}, we use the \emph{DP Batch Perceptron} algorithm presented in Algorithm~\ref{alg:dp-batch-perceptron}. DP Batch Perceptron is the batch and privatized version of the so-called \emph{margin perceptron} algorithm~\citep{DudaH73,CollobertB04}.
%
%
That is, in each iteration, we randomly sample a batch of samples and for each sample $(\bx, y)$ in the batch that is not correctly classified with margin $\gamma'$, we add $y \cdot \bx$ to the current weight of the halfspace. Furthermore, we add some Gaussian noise to the weight vector to make this algorithm private. 
We also have a ``stopping condition'' that terminates whenever the number of samples mislabeled at margin $\gamma'$ is sufficiently small. (We add Laplace noise to the number of such samples to make it private.) To get a $(\gamma, 0.9\gamma)$-robust learner, it suffices for us to set, e.g., $\gamma' = 0.95\gamma$; we use this value of $\gamma'$ in the subsequent discussions.

\begin{figure}[h!]
\centering
\begin{minipage}{0.45\textwidth}
\begin{algorithm}[H]
\begin{algorithmic}[1]
\Procedure{DP-Batch-Perceptron$_{\gamma', p, T, b, \sigma}(\{(\bx_j, y_j)\}_{j \in [n]})$}{}
\State $\bw_0 \leftarrow \bzero$
\For{$i = 1, \dots, T$}
\State $S_i \leftarrow$ a set of samples where each $(\bx_j, y_j)$ is independently included w.p. $p$
\State $M_i \leftarrow \emptyset$
\For{$(\bx, y) \in S_i$}
\If{$\sgn\left(\left<\frac{\bw_{i - 1}}{\|\bw_{i - 1}\|}, \bx\right> - y \cdot \gamma'\right) \ne y$} \label{step:margin-check}
\State $M_i \leftarrow M_i \cup \{(\bx, y)\}$
\EndIf
\EndFor
\State Sample $\nu_i \sim \Lap(b)$
\If{$|M_i| + \nu_i < 0.3\alpha pn$}  \label{step:check-early-stop}
\State \Return $\bw_{i - 1} / \|\bw_{i - 1}\|$ \label{step:return}
\EndIf
\State $\bu_i \leftarrow \sum_{(\bx, y) \in M_i} y \cdot \bx$
\State Sample $\bg_i \sim \cN(0, \sigma^2 \cdot I_{d \times d})$
\State $\bw_i \leftarrow \bw_{i - 1} + \bu_i + \bg_i$
\EndFor
\Return FAIL
\EndProcedure
\end{algorithmic}
\caption{DP Batch Perceptron}
\label{alg:dp-batch-perceptron}
\end{algorithm}
\end{minipage}
\end{figure}

Before we dive into the details of the proof, we remark that our runtime reduction, compared to the generic algorithm from~\citep{BassilyST14}, comes from the fact that, in the accuracy analysis, we only need the number of iterations $T$ to be $O_{\alpha}(1/\gamma^2)$, similar to the perceptron algorithm~\citep{novikoff1963convergence}. On the other hand, the generic theorem of~\cite{BassilyST14} requires $n^2 = \tilde{O}_{\alpha}(d^2/\gamma^2)$ iterations.

The accuracy analysis of DP Batch Perceptron follows the blueprint of that of perceptron~\citep{novikoff1963convergence}. Specifically, we keep track of the following two quantities: $\left<\bw_i, \bw^*\right>$, the dot product between the current halfspace $\bw_i$ and the ``true'' halfspace $\bw^*$, and $\|\bw_i\|$, the (Euclidean) norm of $\bw_i$. We would like to show that, after the first few iterations, $\left<\bw_i, \bw^*\right>$ increases at a faster rate than $\|\bw_i\|$. Since $\left<\bw_i, \bw^*\right>$ is bounded above by $\|\bw_i\|$, we may use this to bound the number $T$ of iterations required for the algorithm to converge.

For simplicity, we assume that in each iteration $|M_i|$ is equal to $m > 0$.
Let us first consider the case where no noise is added (i.e., $\sigma = 0$). From the definition of $\bw^*$, it is simple\footnote{Specifically, since $\cR_{\gamma}(h_{\bw^*}, \cD) = 0$, we have $y \left<\bw^*, \bz\right> \geq 0$ for all $\bz$ such that $\|\bx - \bz\| \leq \gamma$. Plugging in $\bz = \bx - \gamma y \bw^*$ yields the claimed inequality.} to check that $\left<\bw^*, y \cdot \bx\right> \geq \gamma$ for all samples $(\bx, y)$. This means that $\left<\bw^*, \bu_i\right> \geq \gamma m$, resulting in
\begin{align} \label{eq:perceptron-simplified-1}
\left<\bw^*, \bw_i\right> \geq \left<\bw^*, \bw_{i - 1}\right> + \gamma m.
\end{align}
On the other hand, the check condition before we add each example $(\bx, y)$ to $M_i$ (Line~\ref{step:margin-check}) ensures that $\left<\bw_i, y \cdot \bx\right> \leq \gamma' \cdot \|\bw_i\|$ for all $(\bx, y) \in M_i$. From this one can derive the following bound:
\begin{align*}
\|\bw_i\| \leq \|\bw_{i - 1}\| + \gamma' m + \frac{0.5m^2}{\|\bw_{i - 1}\|},
\end{align*}
which, when $\|\bw_{i - 1}\| \geq 50 m/\gamma$, implies that
\begin{align} \label{eq:perceptron-simplified-2}
\|\bw_i\| \leq \|\bw_{i - 1}\| + 0.96 \gamma m.
\end{align}
Combining~\eqref{eq:perceptron-simplified-1} and~\eqref{eq:perceptron-simplified-2}, we arrive at
\begin{align*}
50m/\gamma + 0.96\gamma m i \geq \|\bw_i\| \geq \left<\bw^*, \bw_i\right> \geq \gamma m i.
\end{align*}
This implies that the algorithm must stop after $T = O\left(1/\gamma^2\right)$ iterations. 
(When $m = 1$, this noiseless analysis is essentially the same as the original convergence analysis of perceptron~\citep{novikoff1963convergence}.)

The previous paragraphs outlined the analysis for the noiseless case where $\sigma = 0$. Next, we will describe how the noise $\sigma$ affects the analysis and our choices of parameters. Roughly speaking, we would like the inequalities~\eqref{eq:perceptron-simplified-1} and~\eqref{eq:perceptron-simplified-2} to ``approximately'' hold even after adding noise. In particular, this means that we would like the right-hand side of these inequalities to be affected by at most $o(\gamma m)$ by the noise addition, with high probability. This condition will determine our selection of parameters.

For~\eqref{eq:perceptron-simplified-1}, the inclusion of the noise term $\bg_i$ adds to the right-hand side by $\left<\bw^*, \bg_i\right>$. The expectation of this term is $\|\bw^*\| \cdot \sigma \leq \sigma$, which means that it suffices to ensure that $m \geq \tilde{\omega}\left(\sigma/\gamma\right)$.  For~\eqref{eq:perceptron-simplified-2}, it turns out that the dominant additional term is $\frac{\|\bg_i\|^2}{\|\bw_{i - 1}\|}$ which, under the assumption that $\|\bw_{i - 1}\| \geq 50m/\gamma$, is at most $O(\gamma \|\bg_i\|^2 / m)$; this term is $O(\gamma d \sigma^2 / m)$ in expectation. Since we would like this term to be $o(\gamma m)$, it suffices to have $m = \tilde{\omega}(\sigma \cdot \sqrt{d})$. By combining these two requirements, we may pick $m = \sigma \cdot \tilde{\omega}(\sqrt{d} + 1/\gamma)$. We remark that the number of iterations still remains $T = O(1/\gamma^2)$, as in the noiseless case above.

While we have so far assumed for simplicity that $|M_i| = m$ in all iterations, in the actual analysis we only require that $|M_i| \geq m$. Furthermore, it is simple to show that, as long as the current hypothesis has robust risk significantly more than $\alpha$, we will have $|M_i| \geq \Omega_{\alpha}(pn)$ with high probability. Combining with the previous paragraph, this gives us the following condition (assuming that $\alpha$ is constant):
\begin{align*}
pn \geq \sigma \cdot \tilde{\omega}(\sqrt{d} + 1/\gamma).
\end{align*}
This leads us to pick the following set of parameters: 
\[
n = \tilde{O}\left(\frac{1}{\gamma}\left(\sqrt{d} + \frac{1}{\gamma}\right)\right), 
\quad p = \tilde{O}(\gamma),
\quad \sigma = \tilde{O}(1).
\]

The privacy analysis of our algorithm is similar to that of DP-SGD~\citep{BassilyST14}. Specifically, by the choice of $\sigma = \tilde{O}(1)$ and subsampling rate $p = \tilde{O}(\gamma)$, each iteration of the algorithm is $\left(O(\gamma \epsilon), O(\gamma^2 \delta)\right)$-DP~\citep[e.g.,][]{DworkRV10,BalleBG18}. Since the number of iterations is $T = O(1/\gamma^2)$, advanced composition theorem~\citep{DworkRV10} implies that the entire algorithm is $(O(\sqrt{T} \cdot \gamma \epsilon), O(T \cdot \gamma^2 \delta))$ = $(\eps, \delta)$-DP as desired.

We end by noting that, despite the popularity of perceptron-based algorithms, we are not aware of any work that analyzes the above noised and batched variant. The most closely related analysis we are aware of is that of~\cite{BlumDMN05}, whose algorithm uses the entire dataset in each iteration. While it is possible to adapt their analysis to the batch setting, it unfortunately does not give an optimal sample complexity. Specifically, their analysis requires the  batch size to be $\Omega(\sqrt{d}/\gamma)$, resulting in sample complexity of $\Omega(\sqrt{d}/\gamma^2)$. On the other hand, our more careful analysis works even with batch size $\tilde{O}(\sqrt{d} + 1 / \gamma)$, which results in the desired $\tilde{O_{\alpha, \eps}(\sqrt{d}/\gamma)}$ sample complexity when $d \geq 1 / \gamma^2$.
\section{Experiments}
\label{sec:exp}

\begin{figure*}[ht]
\minipage{0.38\textwidth}
\includegraphics[width=\linewidth]{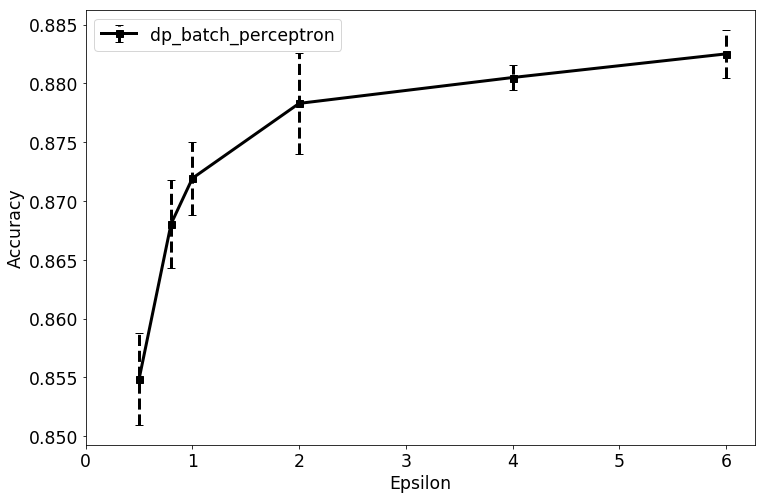}
\endminipage\hspace{-30pt}\hfill
\minipage{0.25\textwidth}
\includegraphics[width=1.1\linewidth]{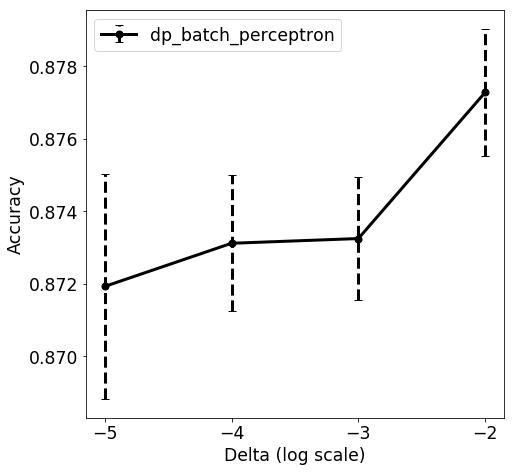}
\endminipage\hspace{-15pt}\hfill
\minipage{0.33\textwidth}
\includegraphics[width=\linewidth]{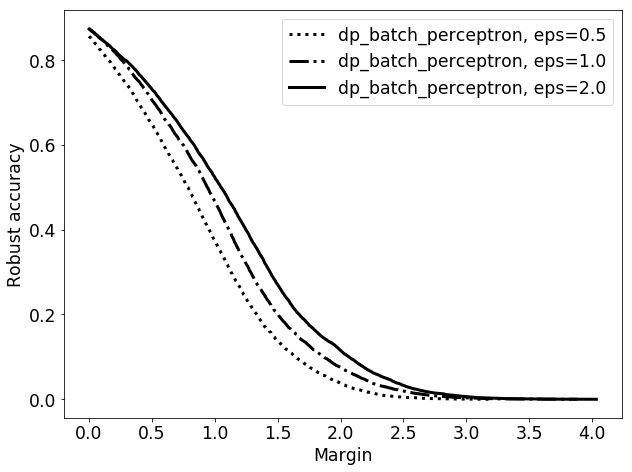}
\endminipage

\minipage{0.38\textwidth}
\includegraphics[width=\linewidth]{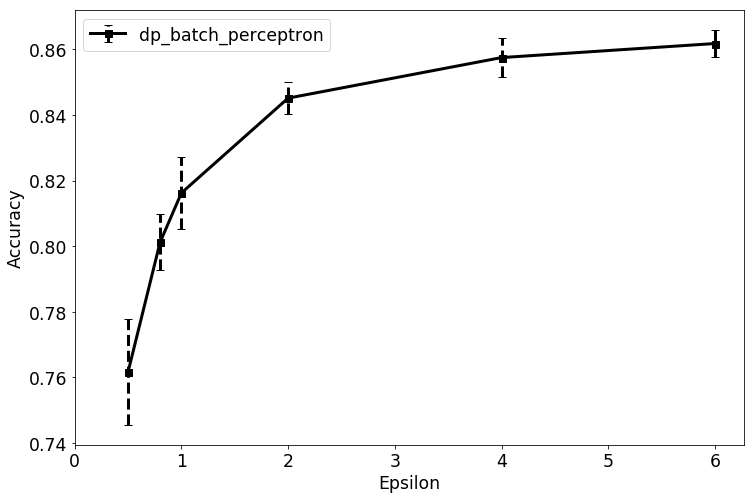}
\subcaption{Accuracy as $\eps$ varies}
\endminipage\hspace{-30pt}\hfill
\minipage{0.25\textwidth}
\includegraphics[width=1.1\linewidth]{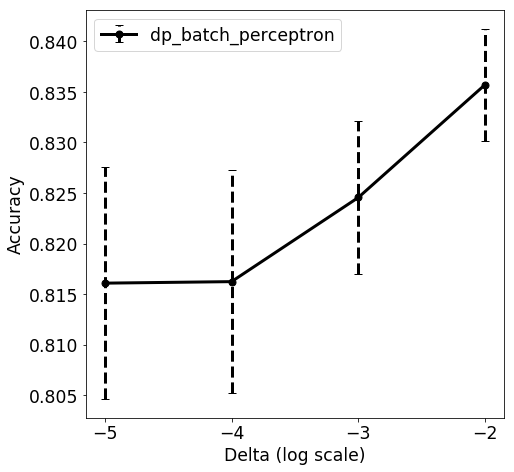}
\subcaption{Accuracy as $\delta$ varies}
\endminipage\hspace{-15pt}\hfill
\minipage{0.33\textwidth}
\includegraphics[width=\linewidth]{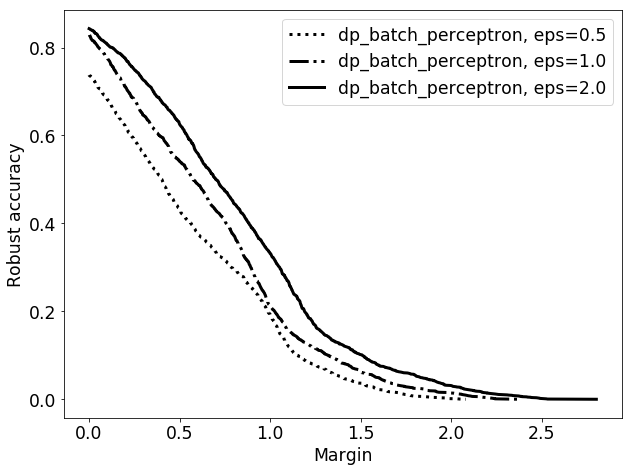}
\subcaption{Robust accuracy as $\eps$ varies}
\endminipage
\caption{Performance of DP Batch Perceptron halfspace classifiers on the MNIST (top row) and USPS (bottom row) datasets.}
\label{fig:mnist_usps}
\end{figure*}

\begin{figure*}[ht]
\minipage{0.33\textwidth}
\includegraphics[width=\linewidth]{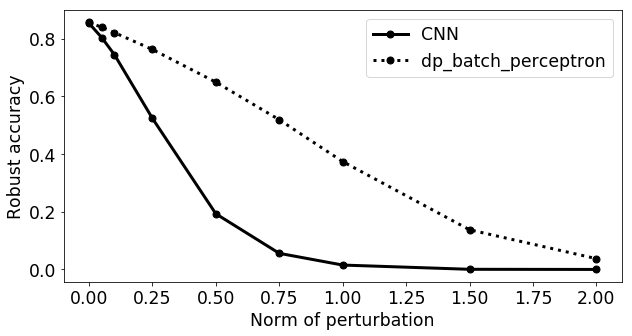}
\subcaption{$\eps = 0.5$}
\endminipage\hspace{-10pt}\hfill
\minipage{0.30\textwidth}
\includegraphics[width=1.1\linewidth]{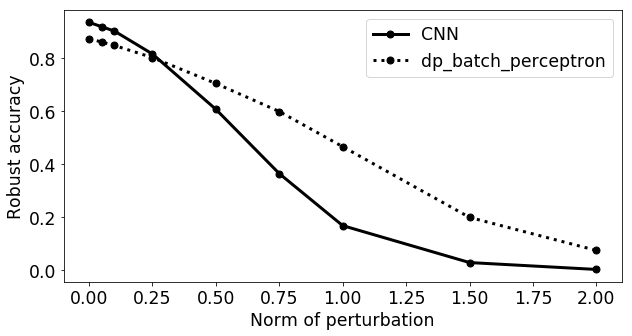}
\subcaption{$\eps = 1$}
\endminipage\hspace{+5pt}\hfill
\minipage{0.33\textwidth}
\includegraphics[width=\linewidth]{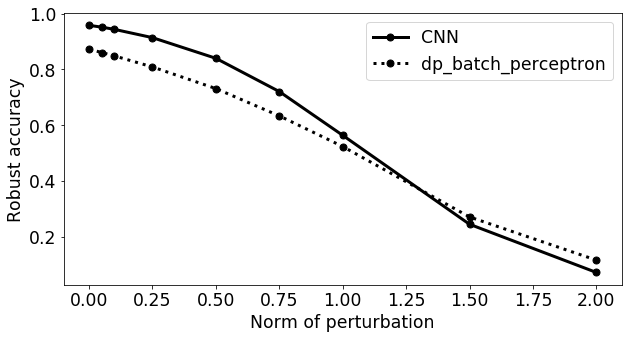}
\subcaption{$\eps = 2$}
\endminipage
\caption{Robustness accuracy comparison between DP-SGD-trained Convolutional neural networks and DP Batch Perceptron halfspace classifiers on MNIST dataset for a fixed privacy budget. In all three plots, $\delta=10^{-5}$ but $\eps$ varies from 0.5, 1, and 2.}
\label{fig:robustness}
\end{figure*}

We run our DP Batch Perceptron algorithm on the MNIST~\citep{lecun2010mnist} and USPS~\citep{hull1994database} datasets, both of which involve 10-class digit classification. We train a separate halfspace classifier $\bw^{(y)}$ for each class $y \in \{1, \dots, 10\}$ for one epoch. To predict on an image $\bx$, we output a class $y^*$ that maximizes $\left<\bw^{(y^*)}, \bx\right>$. 
We tune batch size as a hyperparameter with values 1, 10, 50, 100, 500, 1000, and $\gamma'$ with values 1, 0.1, 0.01, 0.001, 0.0001. 
Each set of experiments is repeated for 20 random trials.
To reduce the number of hyperparameters, we slightly modify our algorithm so that we do not stop early (i.e., removing Lines~\ref{step:check-early-stop} and~\ref{step:return}) but instead return the weight vector $w_T$ at the end of the $T$th iteration (where $T$ is set in the algorithm).

The standard deviation $\sigma$ of the Gaussian noise added is determined by a fixed $(\eps, \delta)$-DP budget, computed using Renyi DP~\citep{abadi2016deep, mironov2017renyi}. The calculations for this follow the implementation in the official TensorFlow Privacy repository (\url{https://github.com/tensorflow/privacy}). For experiments with varying $\eps$ (first column of Figure \ref{fig:mnist_usps}), we fix $\delta$ to $10^{-5}$ for MNIST and $10^{-4}$ for USPS. We observe that despite the robustness and privacy constraints, DP Batch Perceptron still achieves competitive accuracy on both datasets. We also report performance with varying $\delta$ values of $10^{-2}, 10^{-3}, 10^{-4}, 10^{-5}$ (second column of Figure \ref{fig:mnist_usps}), while keeping $\eps$ fixed at 1.0. 

\paragraph{Adversarial Robustness Evaluation.}
We compare the robustness of our models against those of neural networks trained with DP-SGD. For the latter, we follow the architecture found in the official TensorFlow Privacy tutorial, which consists of two convolutional layers, each followed by a MaxPool operation, and a dense layer that outputs predicted logits. The network is then trained with batch size $250$, learning rate $0.15$, $L_2$ clipping-norm $1.0$, and for $60$ epochs. This configuration yields competitive performance on the MNIST dataset.

To evaluate the robustness of the models, we calculate the robust risk on the test dataset for varying values of the perturbation norm (i.e., margin) $\gamma$. Following common practice in the field, we plot the \emph{robust accuracy} on the test data, which is defined as one minus the robust risk (i.e., $1 - \cR_{\gamma}(h, \cD)$ where $\cD$ is the test dataset distribution), instead of the robust risk itself. Similarly, our plots use the \emph{unnormalized} margin, meaning that the images are not re-scaled to have $L_2$ norm equal to $1$ before prediction. Note that each of their pixel values is still scaled (i.e., divided by 255 if necessary) to have values in the range [0, 1].

In the case of DP Batch Perceptron, it is known~\citep[see, e.g.,][]{HeinA17} that an example $(\bx, y)$ cannot be perturbed (using $\bP_{\gamma}$) to an incorrect label if $\gamma < \min_{y' \ne y} \frac{\left<\bw^{(y)}, \bx\right> - \left<\bw^{(y')}, \bx\right>}{\|\bw^{(y)} - \bw^{(y')}\|}$. This formula allows us to exactly calculate the robust risk of our linear classifiers. We stress that this is a \emph{provable} robustness guarantee, i.e., it holds against \emph{all} adversarial attacks with perturbation norm (at most) $\gamma$.

We demonstrate the effect of the change in the required privacy level on the robust risk of our linear classifiers in the right most column of Figure~\ref{fig:mnist_usps}. The x-axis of the plots represents the parameter $\gamma$ and the y-axis represents the $\gamma$-robust accuracy on the test dataset.

In contrast to linear models, there is no efficiently-computable formula to calculate robust risk for general neural networks. In this case, we use a variant of a popular adversarial robustness attack (outlined below) to estimate the robust risk of DP-SGD-trained neural networks. Unlike the linear classifier case, this method only gives a \emph{lower bound} on the robust risk, meaning that more sophisticated attacks might result in even more incorrect classifications. 

We now briefly summarize the attack we use against DP-SGD-trained neural networks; (a version of) this method was already presented in~\citep{szegedy2013intriguing}. Let $M$ denote a trained model; recall that the last layer of our model consists of 10 outputs 
corresponding to each class and to predict an image we take $y^*$ with the maximum output. Given a sample $(\bx, y)$, we would like to determine whether there exists a perturbation $\Delta \in \R^d$ with $\|\Delta\| \leq \gamma$ such that $M$ predict $\bx + \Delta$ to be some other class $y' \ne y$. Instead of solving this (intractable) problem directly, the attack considers a modified objective of
\begin{align*}
\min_{\|\Delta\| \leq \gamma} \ell(M(\bx + \Delta), y')
\end{align*}
where $\ell$ is some loss function. This optimization problem is then solved using Projected Gradient Descent (PGD). 
We use the cross entropy loss in our attack. 

Comparisons of the robust accuracy of models trained via DP Batch Perceptron and those trained via DP-SGD are shown in Figure~\ref{fig:robustness} for $\delta = 10^{-5}$ and $\eps = 0.5, 1, 2$. In the case of $\eps = 0.5$, while both classifiers have similar test accuracies (without any perturbation, $\gamma=0$), as $\gamma$ increases, the robust accuracy rapidly degrades for the DP-SGD-trained neural network compared to that of the DP Batch Perceptron model. This overall trend persists for $\eps = 1, 2$; in both cases, the neural networks start off with noticeably larger test accuracy when $\gamma = 0$ but are eventually surpassed by halfspace classifiers as $\gamma$ increases.
%




\section{Other Related Work}
\label{sec:add-related-work}

\paragraph{Learning Halfspaces with Margin.}

$L_2$ robustness is a classical setting closely related to the notion of margin~\citep[e.g.,][]{rosenblatt1958perceptron, novikoff1963convergence}. (See the supplementary material 
for formal definitions.) Known margin-based learning methods include the classic perceptron algorithm~\citep{rosenblatt1958perceptron, novikoff1963convergence} and its many generalizations and variants ~\citep[e.g.,][]{DudaH73,CollobertB04,FreundS99,GentileL99,LiL02,gentile2001new}, as well as Support Vector Machines (SVMs) \citep{BoserGV92,cortes1995support}. A number of works also explore a closely related \emph{agnostic} setting, where the distribution $\cD$ is not guaranteed to be realizable with a margin~\citep[e.g.,][]{Ben-DavidS00,Shalev-ShwartzSS10,LongS11,BirnbaumS12,DiakonikolasKM19,DKM20}.

Generalization aspects of margin-based learning of halfspaces is also a widely studied topic~\citep[e.g.,][]{Bartlett98,Zhang02,bartlett2002rademacher,koltchinskii2002empirical,McAllester03,KakadeST08}, and it is known that the sample complexity of robust learning of halfspaces is $O(1/(\alpha \gamma)^2)$ \citep{bartlett2002rademacher, koltchinskii2002empirical}.

To the best of our knowledge, the first work that combines the study of learning halfspaces with margin and DP is~\citep{NguyenUZ20}; their results are represented in Table~\ref{table:summary}. Recently,~\cite{GKM20} gave alternative proofs for some results of~\cite{NguyenUZ20} via reductions to clustering problems, but these do not provide any improved sample complexity or running time.

\paragraph{Adversarially Robust Learning.}

There has been a rapidly growing literature on adversarial robustness.
Some of these works have presented evidence that training robust classifiers might be harder than non-robust ones~\citep[e.g.,][]{awasthi2019robustness,bubeck2018adversarial,bubeck2019adversarial,degwekar2019computational}. Other works aim to demonstrate the accuracy cost of robustness~\citep[e.g., ][]{tsipras2018robustness,raghunathan2019adversarial}. Another line of work seeks to determine the right quantity that governs adversarial generalization~\citep[e.g.,][]{schmidt2018adversarially,montasser2019vc,khim2018adversarial,yin2019rademacher,awasthi2020adversarial}.

\paragraph{Differentially Private Learning.}

Private learning has been a popular topic since the early days of differential privacy~\citep[e.g.,][]{KasiviswanathanLNRS08}). Apart from the work of~\cite{NguyenUZ20} on privately learning halfspaces with a margin, a line of work closely related to our setting is the study of the sample complexity of learning threshold functions~\citep{BeimelNS16,FeldmanX14,BunNSV15,AlonLMM19,kaplan_privately_2020} and halfspaces~\citep{beimel_private_2019,kaplan_private_2020,kaplan_how_2020}. These works study the setting where the unit ball $\bB^d$ is discretized so that the domain $\cX$ is $X^d \cap \bB^d$ (i.e., each coordinate is an element of $X$). Interestingly, it has been shown that when $X$ is infinite, halfspaces become unlearnable, i.e., the sample complexity becomes unbounded~\citep{AlonLMM19}. On the other hand, an $(\eps, o(1/n))$-DP learner with sample complexity $\tilde{O}\left(\frac{d^{2.5}}{\alpha \eps}\right) \cdot 2^{O(\log^* |X|)}$ exists~\citep{kaplan_private_2020}.

While the above setting is not directly comparable to ours, it is possible to reduce between the margin setting and the discretized setting, albeit with some loss. For example, we may use the grid discretization with $|X| = 0.01\gamma / \sqrt{d}$, to obtain a $(\gamma, 0)$-learner with sample complexity $\tilde{O}\left(\frac{d^{2.5}}{\alpha \eps}\right) \cdot 2^{O(\log^*(1/\gamma))}$. This is better than the (straightforward) bound of $O(d \cdot \log(1/\gamma))$ obtained by applying the exponential mechanism~\citep{McSherryT07} when $\gamma$ is very small (e.g., $\gamma \leq 2^{-d^{1.6}}$). It remains an interesting open problem to close such a gap for very small values of $\gamma$.

Several works~\citep[e.g.,][]{rubinstein2009learning,chaudhuri2011differentially} have studied differentially private SVMs. However, to the best of our knowledge, there is no straightforward way to translate their theoretical results to those shown in our paper as the objectives in the two settings are different.

\section{Conclusions and Future Directions}\label{sec:disc_open_questions}
In this work, we prove new trade-offs, measured in terms of sample complexity, between privacy and robustness---two crucial properties within the domain of AI ethics and safety---for the classic task of halfspace learning. Our theoretical results demonstrate that DP \emph{and} adversarially robust learning requires a larger number of samples than either DP \emph{or} adversarially robust learning alone. We then propose a learning algorithm that meets both criteria, and test it on two multi-class classification datasets. We also provide empirical evidence that despite having a slight advantage in terms of test accuracy on the main task, standard neural networks trained with DP-SGD are not as robust as those trained with our algorithm.

We conclude with a few future research directions. First, it would be interesting to close the gap for the sample complexity of \emph{improper} approximate-DP $(\gamma, 0.9\gamma)$-robust learners; for $d \gg 1/\gamma^2$, the upper bound is $O(\sqrt{d}/\gamma)$ (Theorem~\ref{thm:perceptron-apx-DP-simplified}) but the lower bound is only $\Omega(1/\gamma^2)$ (Theorem~\ref{thm:lower-bound-non-robust}). This is the only case where there is still a super-polylogarithmic gap for $d \gg 1/\gamma^2$. 

Another technical open question is to improve the lower bounds in Theorem~\ref{thm:lower-bound-apx-proper} to $\Omega(\min\{\sqrt{d}/\gamma, d\}) / \eps$. Currently, we are missing the $\eps$ term because we invoke a lower bound from~\cite{SteinkeU17} (Theorem~\ref{thm:fingerprinting}), which was specifically proved only for $\eps = 1$.

Furthermore, it would be natural to extend our study to $L_p$ perturbations for $p \ne 2$. An especially noteworthy case is when $p = \infty$, which is a well-studied setting in the adversarial robustness literature.

Finally, it would be very interesting to provide a theoretical understanding of private and robust learning beyond halfspaces, to accommodate complex algorithms (e.g., deep neural networks) that are better suited for more challenging tasks.

\balance
\bibliographystyle{plainnat}
\bibliography{refs}

\newpage
\onecolumn

\begin{center}
{\LARGE \bf Supplementary Material}
\end{center}
\appendix
\section{Preliminaries}

For $m \in \N$, we use $[m]$ to denote $\{1, \dots, m\}$. For a distribution $\cD$, we write $r \sim \cD$ to denote a random variable $r$ distributed as $\cD$. For a randomized algorithm $\A$, we write $\A(\bX)$ to denote the distribution of the output of $\A$ on input $\bX$. For a distribution $\cD$, we write $\A(\cD)$ to denote the distribution of the output of $\A$ when the input is drawn from $\cD$. Sometimes we will allow the number of samples drawn by an algorithm to be a random variable. In this case, the algorithm must specify the number of samples before seeing any samples. Furthermore, when $\cD$ is the distribution of each sample, we may write $\A_{\cD}$ to denote the distribution of the output when each of $\A$'s samples is drawn from $\cD$. We use $\cD_1 \otimes \cdots \otimes \cD_m$ to denote the product distribution of the distributions $\cD_1, \dots, \cD_m$. Furthermore, we use $\cD^{\otimes m}$ to denote the $m$-fold product of the distribution $\cD$ with itself.
 
For convenience, we interchangeably refer to a halfspace by $h_{\bw}$ or just the weight vector $\bw$ itself.

\subsection{Margin of Halfspaces}
\label{subsec:margin-prelim}

Robust learning of halfspaces is intimately related to the notion of margin. For a margin parameter $\gamma > 0$, we say that an example $(\bx, y) \in \R^d \times \{\pm 1\}$ is \emph{correctly classified by $\bw$ with margin $\gamma$} iff $\sgn(\left<\bw, \bx\right> - y \cdot \gamma) = y$. The $\gamma$-margin error is defined as $\err_{\gamma}^{\cD}(\bw) = \Pr_{(\bx, y) \sim \mathcal{D}}[\sgn(\left<\bw, \bx\right> - y \cdot \gamma) \ne y]$.
The connection between robust learning of halfspaces and learning with margin is given through the following (folklore) lemma; its proof can be found, e.g., in~\citep{DKM20}.

\begin{lemma}
For any non-zero $\bw \in \R^d$, $\gamma \geq 0$ and $\cD$, $\cR_{\gamma}(\bw, \cD) = \err^{\cD}_\gamma\left(\frac{\bw}{\|\bw\|_2}\right)$.
\end{lemma}

Due to the above lemma, we may refer to the $\gamma$-margin error for halfspaces instead of their robust risk throughout the paper. 

\subsection{Boosting the Success Probability}

Throughout this work, it is often more convenient to prove lower bounds (resp. upper bounds) only for some large (resp. small) failure probability $\xi \in (0, 1)$. We note that this is without loss of generality, since standard techniques can be used to boost the success probability while incurring small loss in the sample complexity. We sketch the argument below.

\begin{observation}
For any $\xi, \xi' \in (0, 1)$, the following statement holds:
If there is a $(\gamma, \gamma')$-robust learner with failure probability $\xi$, accuracy $\alpha$ and sample complexity $m$, then there exists a $(\gamma, \gamma')$-robust learner with failure probability $\xi'$, accuracy $1.1\alpha$ and sample complexity $O_{\xi, \xi'}(m + 1/\alpha^2)$.
\end{observation}

\begin{proof}[Proof Sketch]
Let $\A$ be the $(\gamma, \gamma')$-robust learner with failure probability $\xi$, accuracy $\alpha$ and sample complexity $m$. We define an algorithm $\bB$ as follows:
\begin{itemize}
\item Let $T := \lceil \frac{\log(0.5\xi')}{\log(1 - \xi)} \rceil$ and $M := \lceil \frac{10^6 \cdot \log T}{\alpha} \rceil$.
\item For $i \in [T]$, run $\A$ on $m$ samples to get a halfspace $\bw_i$.
\item Sample $M$ fresh new samples. Then, output $\bw_i$ that minimizes the $\gamma$-margin error of $\bw_i$ on the uniform distribution over these $M$ samples.
\end{itemize}
Clearly, the algorithm $\bB$ uses $m \cdot T + M = O_{\xi, \xi'}(m + 1/\alpha^2)$ samples as desired. For the accuracy, with probability $1 - (1 - \xi)^T \geq 1 - 0.5\xi'$ at least one of the $\bw_i$'s satisfies $\err^{\cD}_{\gamma'}(\bw_i) \leq \alpha$. Conditioned on this, the Chernoff bound ensures that w.p. $1 - 0.5\xi$ we output a $\bw_i$ s.t. $\err^{\cD}_{\gamma'}(\bw_i) \leq 1.1\alpha$. We can then conclude the proof via the union bound.
\end{proof}
\section{Lower Bound for Robust Learning of Halfspaces: Pure-DP Case}

In this section, we prove our lower bound for $\eps$-DP robust learning of halfspaces (Theorem~\ref{thm:lower-bound}), which is restated below.

\lowerbound*

We will use the following (well-known) fact; for completeness, we sketch its proof at the end of this section.

\begin{lemma} \label{lem:codes}
There exist $\bw^{(1)}, \dots, \bw^{(K)} \in \R^d$ where $K = 2^{\Omega(d)}$ such that $\|\bw^{(i)}\|_2 = 1$ for all $i \in [K]$ and $|\left<\bw^{(i)}, \bw^{(j)}\right>| < 0.01$ for all $i \ne j$.
\end{lemma}

\begin{proof}[Proof of Theorem~\ref{thm:lower-bound}]
We will prove the statement for any $\gamma \leq 0.99, \alpha \leq 0.49$ and $\xi \leq 0.9$.

Let $\bw^{(1)}, \dots, \bw^{(K)}$ be the vectors guaranteed by Lemma~\ref{lem:codes}. For each $i \in [K]$, we define $\cD^{(i)}$ to be the uniform distribution on two elements: $(1.01\gamma \cdot \bw^{(i)}, +1)$ and $(-1.01\gamma \cdot \bw^{(i)}, -1)$. Notice that $\mathcal{R}_{\gamma}(\bw^{(i)}, \cD^{(i)}) = 0$.

Now, let $G^{(i)} = \{h: \bB^d \to \{\pm 1\} \mid \mathcal{R}_{0.9\gamma}(h, \cD^{(i)}) \leq \alpha\}$ denote the set of hypotheses which incurs error no more than $\alpha$ on $\cD^{(i)}$. The main claim is the following:

\begin{claim} \label{claim:disjoint-answer}
For every $i \ne j$, $G^{(i)} \cap G^{(j)} = \emptyset$.
\end{claim}

\begin{proof}
Suppose for the sake of contradiction that there exists $h \in G^{(i)} \cap G^{(j)}$ for some $i \ne j$.

Since $\alpha \leq 0.49$ and $\cD^{(i)}$ is a uniform distribution over only two samples, $\mathcal{R}_{0.9\gamma}(h, \cD^{(i)}) \leq \alpha$ implies that $\mathcal{R}_{0.9\gamma}(h, \cD^{(i)}) = 0$. This implies that
\begin{align*}
h(z) = 1 & & \forall z \in \bP_{0.9\gamma}(1.01\gamma \cdot \bw^{(i)}).
\end{align*}

By an analogous argument, we have
\begin{align*}
h(z) = -1 & & \forall z \in \bP_{0.9\gamma}(-1.01\gamma \cdot \bw^{(j)}).
\end{align*}
This is a contradiction since $\bP_{0.9\gamma}(-1.01\gamma \cdot \bw^{(j)}) \cap \bP_{0.9\gamma}(1.01\gamma \cdot \bw^{(i)}) \ne \emptyset$; specifically, $|\left<\bw^{(i)}, \bw^{(j)}\right>| < 0.01$ implies that this intersection contains $0.505 \gamma \cdot \bw^{(i)} - 0.505 \gamma \cdot \bw^{(j)}$.
\end{proof}

To finish the proof, consider any $\eps$-DP $(\gamma, 0.9\gamma)$-robust learner $\A$ with $\alpha \leq 0.49$. Suppose that it takes $n$ samples. Notice that, when we feed it $n$ random samples from $\cD^{(i)}$, the accuracy guarantee ensures that
\begin{align*}
\Pr_{(\bx_1, y_1), \dots, (\bx_n, y_n) \sim \cD^{(i)}}[\A((\bx_1, y_1), \dots, (\bx_n, y_n)) \in G^{(i)}] \geq 1 - \xi.
\end{align*}
As a result, since $\A$ is $\eps$-DP, we have
\begin{align}
\Pr[\A(\emptyset) \in G^{(i)}] \geq (1 - \xi) \cdot e^{-\eps \cdot n} \geq 0.1 \cdot e^{-\eps \cdot n}. \label{eq:packing-single-bound}
\end{align}
From Claim~\ref{claim:disjoint-answer}, $G^{(1)}, \dots, G^{(K)}$ are disjoint, which implies
\begin{align*}
1 \geq \sum_{i \in [K]} \Pr[\A(\emptyset) \in G^{(i)}] \overset{\eqref{eq:packing-single-bound}}{\geq} K \cdot 0.1 \cdot e^{-\eps \cdot n}.
\end{align*}
Thus, we have $n \geq \Omega\left(\frac{\log K}{\eps}\right) = \Omega(d/\eps)$ as desired.
\end{proof}

Finally, we briefly sketch the proof of Lemma~\ref{lem:codes}.

\begin{proof}[Proof of Lemma~\ref{lem:codes}]
It is well-known that there exist linear error correcting codes over $\mathbb{F}_2$ with constant rate and distance 0.4995. (See e.g.~\cite[Section 7]{AlonGHP90} for an explanation.) Equivalently, this means that there exists a linear space $V \subseteq \mathbb{F}_2^d$ of dimension $\Omega(d)$ such that $\|\bv\|_0 \in [0.4995d, 0.5005d]$ for all non-zero $\bv \in V$ where $\|\cdot\|_0$ denote the Hamming norm (i.e. number of non-zero coordinates).

Let $\bv^{(1)}, \dots, \bv^{(K)}$ denote the elements of $V$ notice that $K = 2^{\dim(V)} = 2^{\Omega(d)}$. Define $\bw^{(1)}, \dots, \bw^{(K)} \in \R^d$ where 
\begin{align*}
\bw^{(i)}_\ell =
\begin{cases}
-1/\sqrt{d} & \text{ if  } \bv^{(i)}_\ell = 0 \\
+1/\sqrt{d} & \text{ if  } \bv^{(i)}_\ell = 1.
\end{cases} 
\end{align*}
For $i \ne j$, we have
\begin{align*}
\left|\left<\bw^{(i)}, \bw^{(j)}\right>\right| = |1 - 2 \cdot \|\bv^{(i)} - \bv^{(j)}\|_0 / d| \leq 0.01d,
\end{align*}
where the latter follows from linearity of $V$. This concludes our proof.
\end{proof}
\section{Lower Bound for Robust Learning of Halfspaces: Approximate-DP Case}

For our lower bound for approximate-DP proper learners (Theorem~\ref{thm:lower-bound-apx-proper}), we will reduce from a lower bound of \cite{SteinkeU17}. To state their results, we will need some additional notation. 
Let $\cU_{[0, 1]}$ denote the uniform distribution on $[0, 1]$, and let $\cB_q$ denote the distribution that is $+1/\sqrt{d}$ with probability $q$ and is $-1/\sqrt{d}$ otherwise. For $\bq \in [0, 1]^d$, we use $\cB_\bq$ to denote $\cB_{q_1} \otimes \cdots \otimes \cB_{q_d}$. \cite{SteinkeU17} prove the following theorem\footnote{We remark that (1) the result of~\cite{SteinkeU17} is stated for the Beta distributions which contain the uniform distribution (i.e., $\cU([0, 1]) = \Beta(1, 1)$) (2) we scale down the output $\cM(\bX)$ by a factor of $1/\sqrt{k}$ (which has the same effect on the error), (3) we replace the Bernoulli distribution with $\cB_q$ which is valid since there is a one-to-one mapping between the two, (4) the theorem of~\cite{SteinkeU17} has another parameter $k$ which we simply set to $d$ and (5) the original theorem in~\cite{SteinkeU17} implicitly imposes a bound on $\|\cM(\bX)\|_{\infty}$ but the actual condition needed is on $\|\cM(\bX)\|_1$ which is already implied by our condition that $\|\cM(\bX)\|_2 \leq 1$.}\footnote{We also remark that a similar theorem can already be derived via the work of \cite{DworkSSUV15}; however, we choose to state this version since it is more compatible with our reduction and is readily available already in~\citep{SteinkeU17}.}:

\begin{theorem}[{\cite[Theorem 3]{SteinkeU17}}] \label{thm:fingerprinting}
Let $\zeta > 0$ and $n, d \in \N$ be such that $n < \zeta \sqrt{d}$. Let $\cM$ be any $(1, \zeta/n)$-DP algorithm whose output belongs to the $d$-dimensional unit Euclidean ball. Let $\bX = (\bx_1, \dots, \bx_n)$ be such that $\bx_i$ is i.i.d. drawn from $\cB_{\bq}$. Then,
\begin{align} \label{eq:fingerprinting-correlation-upper-bound}
\E_{\bq \sim \cU_{[0, 1]}^{\otimes d}, \bw \sim  \cM(\cB_{\bq}^{\otimes n})}\left[\sum_{j \in [d]} \bw_j \cdot \left(q_j - 0.5\right)\right] < \zeta \sqrt{d}.
\end{align}
\end{theorem}

In the next subsection, we first show a lower bound of $\Omega(\sqrt{d})$ for any sufficiently small constant $\gamma$ (Lemma~\ref{lem:distributional-constant-gamma-hardness}). Then, in Subsection~\ref{subsec:apx-dp-lb-small-gamma}, we use this to prove a lower bound of $\Omega(\min\{\sqrt{d}/\gamma, d\})$.

\subsection{Lower Bound for $\gamma = \Omega(1)$}

We cannot use the distribution $\cB_\bq$ directly since it is not realizable with a large margin. To overcome this, we define $\cP_\bq$ as the distribution of $\bx \sim \cB_{\bq}$ conditioned on $\left<\bq', \bx\right> \geq 0.01$ where we write $\bq'$ as a shorthand for $\frac{1}{\sqrt{d}}\left(2\bq - \bone\right)$. We will require the following bound:

\begin{lemma} \label{lem:tv-margin}
$\E_{\bq \sim \cU_{[0, 1]}^{\otimes d}}\left[d_{TV}(\cB_{\bq}, \cP_{\bq})\right] \leq o(1/d)$.
\end{lemma}

\begin{proof}
The Chernoff bound implies that $\Pr_{\bq \sim \cU_{[0, 1]}^{\otimes d}}[\|\bq'\| \geq 0.1] \geq o(1/d)$. For a fixed $\bq$ such that $\|\bq'\| \geq 0.1$, the Chernoff bound again yields that $\Pr_{\bx \sim \cB_{\bq}}[\left<\bq', \bx\right> \geq 0.01] \leq o(1/d)$, which implies that $d_{TV}(\cB_{\bq}, \cP_{\bq}) \leq o(1/d)$. Combining these, we have $\E_{\bq \sim \cU_{[0, 1]}^{\otimes d}}\left[d_{TV}(\cB_{\bq}, \cP_{\bq})\right] \leq o(1/d)$ as desired.
\end{proof}

Let $\tcP_{\bq}$ denote the distribution of $(\bx, +1)$ where $\bx \sim \cP_{\bq}$. Similarly, let $\tcB_{\bq}$ denote the distribution of $(\bx, +1)$ where $\bx \sim \cB_{\bq}$.
We can now prove our $\Omega(\sqrt{d})$ lower bound for any sufficiently small constant $\gamma > 0$, which follows almost immediately from the following lemma.

\begin{lemma} \label{lem:distributional-constant-gamma-hardness}
For any constant $\gamma, \beta \in (0, 1)$ such that $\gamma > 2\beta$, the following holds.
Let $\A$ be any $(1, o(1/n))$-DP algorithm with sample complexity $n$ and whose output belongs to the $d$-dimensional unit Euclidean ball. If
\begin{align*}
\E_{\bq \sim \cU_{[0, 1]}^{\otimes d}, \bw \sim \A(\tcP_\bq^{\otimes n})}[\err_{\gamma}^{\tcP_\bq}(\bw)] \leq \beta,
\end{align*}
then we must have $n \geq \Omega(\sqrt{d})$.
\end{lemma}

\begin{proof}
Suppose for the sake of contradiction that there exists a $(1, o(1/n))$-DP algorithm $\A$ with sample complexity $n = o(\sqrt{d})$ whose output is a $d$-dimensional vector of Euclidean norm at most one that satisfies $\E_{\bq \sim \cU_{[0, 1]}^{\otimes d}, \bw \sim \A(\cP_\bq^{\otimes n})}[\err_{\gamma}^{\cP_\bq}(\bw)] \leq \beta$.

On input $\bx_1, \dots, \bx_n \in \{\pm 1/\sqrt{d}\}^d$, $\cM$ simply works as follows: Run $\A$ on $(\bx_1, +1), \dots, (\bx_n, +1)$ to obtain a halfspace $\bw$ and output $\bw$. Now, we have that
\begin{align}
&\E_{\bq \sim \cU_{[0, 1]}^{\otimes d}, \bw \sim  \cM(\cB_{\bq}^{\otimes n})}\left[\sum_{j \in [d]} \bw_j \cdot \left(q_j - 0.5\right)\right] \nonumber \\
&= \E_{\bq \sim \cU_{[0, 1]}^{\otimes d}, \bw \sim \A(\tcB_{\bq}^{\otimes n})}\left[\sum_{j \in [d]} \bw_j \cdot \left(q_j - 0.5\right)\right] \nonumber \\
&\geq \E_{\bq \sim \cU_{[0, 1]}^{\otimes d}, \bw \sim \A(\tcP_{\bq}^{\otimes n})}\left[\sum_{j \in [d]} \bw_j \cdot \left(q_j - 0.5\right)\right] - \E_{\bq \sim \cU_{[0, 1]}^{\otimes d}}\left[d_{TV}(\tcP_{\bq}^{\otimes n}, \tcB_{\bq}^{\otimes n}) \cdot (0.5\sqrt{d})\right] \nonumber \\
&\geq \E_{\bq \sim \cU_{[0, 1]}^{\otimes d}, \bw \sim \A(\tcP_{\bq}^{\otimes n})}\left[\sum_{j \in [d]} \bw_j \cdot \left(q_j - 0.5\right)\right] - (0.5 n\sqrt{d}) \cdot \E_{\bq \sim \cU_{[0, 1]}^{\otimes d}}\left[d_{TV}(\cP_{\bq}, \cB_{\bq})\right] \nonumber \\
&\overset{\text{Lemma}~\ref{lem:tv-margin}}{\geq} \E_{\bq \sim \cU_{[0, 1]}^{\otimes d}, \bw \sim \A(\tcP_{\bq}^{\otimes n})}\left[\sum_{j \in [d]} \bw_j \cdot \left(q_j - 0.5\right)\right] - o(n/\sqrt{d}) \nonumber \\
&= \E_{\bq \sim \cU_{[0, 1]}^{\otimes d}, \bw \sim \A(\tcP_{\bq}^{\otimes n})}\left[\sum_{j \in [d]} \bw_j \cdot \left(q_j - 0.5\right)\right] - o(1), \label{eq:conditioned-to-halfspace}
\end{align}
where in the first inequality we use the fact that $\|\bw\|_2 \leq 1$, which implies that $\|\bw\|_1 \leq \sqrt{d}$.

Notice that we may rearrange the term inside the expectation in~\eqref{eq:conditioned-to-halfspace} as follows:
\begin{align}
\sum_{j \in [d]} \bw_j \cdot (q_j - 0.5)
&= \frac{\sqrt{d}}{2} \sum_{j \in [d]} \bw_j \cdot \frac{2q_j - 1}{\sqrt{d}} \nonumber \\
&= \frac{\sqrt{d}}{2} \left<\bw, \bq'\right> \nonumber \\
&= \frac{\sqrt{d}}{2} \left<\bw, \E_{\bx \sim \cB_\bq}[\bx]\right> \nonumber \\
&= \frac{\sqrt{d}}{2} \cdot \E_{\bx \sim \cB_\bq}[\left<\bw, \bx\right>] \nonumber \\
&\overset{\text{Lemma}~\ref{lem:tv-margin}}{\geq} \frac{\sqrt{d}}{2} \cdot \E_{\bx \sim \cP_\bq}[\left<\bw, \bx\right>] - o(1) \nonumber \\
&\geq \frac{\sqrt{d}}{2} \left(\gamma \cdot \Pr_{\bx \sim \cP_\bq}[\left<\bw, \bx\right> \geq \gamma] - 1 \cdot \Pr[\left<\bw, \bx\right> < \gamma]\right) - o(1) \nonumber \\
&= \frac{\sqrt{d}}{2} \left(\gamma \cdot (1 - \err^{\tcP_{\bq}}_\gamma(\bw)) - \err^{\tcP_{\bq}}_\gamma(\bw)\right) - o(1) \nonumber \\
&\geq \frac{\sqrt{d}}{2} \left(\gamma - 2 \err^{\tcP_{\bq}}_\gamma(\bw)) \right) - o(1) \nonumber 
\end{align}

Plugging this back into~\eqref{eq:conditioned-to-halfspace}, we have that
\begin{align*}
&\E_{\bq \sim \cU_{[0, 1]}^{\otimes d}, \bw \sim  \cM(\cB_{\bq}^{\otimes n})}\left[\sum_{j \in [d]} \bw_j \cdot \left(q_j - 0.5\right)\right] \\ &\geq \frac{\sqrt{d}}{2} \cdot \left(\gamma - 2 \cdot \E_{\bq \sim \cU_{[0, 1]}^{\otimes d}, \bw \sim \A(\tcP_{\bq}^{\otimes n})}\left[\err^{\tcP_{\bq}}_\gamma(\bw))\right]\right) - o(1) \\
&\geq \frac{\sqrt{d}}{2}\left(\gamma - 2 \beta\right) - o(1) \\
&= \Omega(\sqrt{d}),
\end{align*}
which contradicts Theorem~\ref{thm:fingerprinting}. This concludes our proof.
\end{proof}

\subsection{Lower Bound for Smaller $\gamma$}
\label{subsec:apx-dp-lb-small-gamma}

We will now reduce from the case $\gamma = \Omega(1)$ to get a larger lower bound for smaller $\gamma$. To do this, it will be convenient to have an ``expected version'' of Lemma~\ref{lem:distributional-constant-gamma-hardness}, which is stated and proved below.

\begin{lemma} \label{lem:distributional-constant-gamma-hardness-convenient}
For any constants $\gamma_0, \beta_0 \in (0, 1)$ such that $\gamma_0 > 4 \sqrt{2\beta_0}$, the following holds.
Let $\bB$ be any $(1, o(1/n))$-DP algorithm that has access to an oracle $\cO$ that can sample from $\tcP_{\bq}$ where $\bq$ is unknown to $\bB$. All of the following cannot hold simultaneously:
\begin{enumerate}
\item The expected number of samples $\bB$ draws from $\cO$ is $o(\sqrt{d})$.
\item $\E_{\bq \sim \cU_{[0, 1]}^{\otimes d}, \bw \sim \bB_{\tcP_{\bq}}}[\|\bw\|^2] \leq 1$.
\item $\E_{\bq \sim \cU_{[0, 1]}^{\otimes d}, \bw \sim \bB_{\tcP_{\bq}}}[\err^{\tcP_\bq}_{\gamma_0}(\bw)] \leq \beta_0$.
\end{enumerate}
\end{lemma}

\begin{proof}
Suppose for the sake of contradiction that there exists a $(1, o(1/n))$ algorithm $\bB$ that draws $o(\sqrt{d})$ samples from $\cO$ in expectation, and satisfies $\E_{\bq \sim \cU_{[0, 1]}^{\otimes d}, \bw \sim \bB_{\tcP_{\bq}}}[\|\bw\|^2] \leq 1$ and $\E_{\bq \sim \cU_{[0, 1]}^{\otimes d}, \bw \sim \bB_{\tcP_{\bq}}}[\err^{\tcP_\bq}_{\gamma_0}(\bw)] \leq \beta_0$. We use $\bB$ to construct an algorithm $\A$ that will contradict  Lemma~\ref{lem:distributional-constant-gamma-hardness} as follows:
\begin{itemize}
\item Run $\bB$.
\item If $\bB$ attempts to take more than $2n / \beta_0$ sample, simply output $\bzero$.
\item Otherwise, let $\bw$ be the output of $\bB$, and output $\bw' = \frac{\bw}{\|\bw\|}$. 
\end{itemize}
Notice that $\A$ is $(1, o(1/n))$-DP and the number of samples used is $2n / \beta_0 = o(\sqrt{d})$.

Let $\beta = 2\beta_0$ and $\gamma = \gamma_0 \sqrt{\beta_0/2}$.
We will next argue that $\E_{\bq \sim \cU_{[0, 1]}^{\otimes d}, \bw' \sim \A(\tcP_\bq^{\otimes n})}[\err_{\gamma}^{\tcP_\bq}(\bw')] \leq \beta$. First, since the expected number of samples of $\bB$ is $n$, by Markov's inequality, the probability that $\bB$ takes more than $2n / \beta_0$ samples is at most $\beta_0 / 2$. As a result, we have that
\begin{align*}
\E_{\bq \sim \cU_{[0, 1]}^{\otimes d}, \bw' \sim \A(\tcP_\bq^{\otimes n})}[\err_{\gamma}^{\tcP_\bq}(\bw')] \leq \E_{\bq \sim \cU_{[0, 1]}^{\otimes d}, \bw \sim \B_{\tcP_\bq^{\otimes n}}}[\err_{\gamma}^{\tcP_\bq}(\bw / \|\bw\|)] + \beta_0 / 2.
\end{align*}
Recall also that $\E_{\bq \sim \cU_{[0, 1]}^{\otimes d}, \bB}[\|\bw\|^2] \leq 1$; Markov's inequality once again implies that $\Pr_{\bq \sim \cU_{[0, 1]}^{\otimes d}, \bB}[\|\bw\|^2 > 2/\beta_0] \leq \beta_0/2$. Plugging this into the above inequality, we get that
\begin{align*}
\E_{\bq \sim \cU_{[0, 1]}^{\otimes d}, \bw' \sim \A(\tcP_\bq^{\otimes n})}[\err_{\gamma}^{\tcP_\bq}(\bw')]
&\leq \E_{\bq \sim \cU_{[0, 1]}^{\otimes d}, \bw \sim \B_{\tcP_\bq^{\otimes n}}}[\err_{\gamma}^{\tcP_\bq}(\bw / \|\bw\|) \cdot \1[\|\bw\| \leq \sqrt{2/\beta_0}]] + \beta_0. \\
&\leq \E_{\bq \sim \cU_{[0, 1]}^{\otimes d}, \bw \sim \B_{\tcP_\bq^{\otimes n}}}[\err_{\gamma}^{\tcP_\bq}(\bw / \sqrt{2 / \beta_0})] + \beta_0 \\
&= \E_{\bq \sim \cU_{[0, 1]}^{\otimes d}, \bw \sim \B_{\tcP_\bq^{\otimes n}}}[\err_{\gamma_0}^{\tcP_\bq}(\bw)] + \beta_0 \\
(\text{From our third assumption on } \bB) &\leq 2\beta_0 = \beta,
\end{align*}
which is a contradiction to Lemma~\ref{lem:distributional-constant-gamma-hardness} since $\gamma > 2\beta$.
\end{proof}

We can now prove our $\Omega(\min\{\sqrt{d}/\gamma, d\})$ lower bound (Theorem~\ref{thm:lower-bound-apx-proper}). Roughly speaking, when $d \geq 1/\gamma^2$, we ``embed'' $\Theta(1/\gamma^2)$ hard distributions from Lemma~\ref{lem:distributional-constant-gamma-hardness-convenient} into $\Theta(\gamma^2 d)$ dimensions, which results in the $\Omega(\sqrt{\gamma^2 d} \cdot 1/\gamma^2) = \Omega\left(\sqrt{d}/\gamma\right)$ lower bound as desired.

\lowerboundapx*

\begin{proof}
We will prove this lower bound for $\gamma \leq 0.01, \alpha, \xi \leq 10^{-6}$.

First, notice that, when $\gamma \leq 1/\sqrt{d}$, a $(\gamma, 0.9\gamma)$-robust proper learner is also an $(1/\sqrt{d}, 0)$-robust proper learner. Hence, by Theorem~\ref{thm:lower-bound-non-robust}, we have $n = \Omega(d)$ as desired. Thus, we can subsequenly only focus on the case $\gamma \geq 1/\sqrt{d}$, for which we will show that $n = \Omega(\sqrt{d}/\gamma)$.

Suppose for the sake of contradiction that there is a $(1, o(1/n))$-DP $(\gamma, 0.9\gamma)$-robust proper learner $\A$ with $\alpha, \xi \leq 10^{-6}$ that has sample complexity $n = o(\sqrt{d} / \gamma)$. Let $T = \lfloor 0.01 / \gamma\rfloor$, and $d' = \lfloor d / T^2 \rfloor$. We will construct an algorithm $\bB$ that contradicts with Lemma~\ref{lem:distributional-constant-gamma-hardness-convenient} in $d'$ dimensions.

We will henceforth assume w.l.o.g. that $d = d' \cdot T^2$. This is without loss of generality since the proof below extends to the case $d > d' \cdot T^2$ by padding $d - d' \cdot T^2$ zeros to each of the samples.

In the following, we view the $d$-dimensional space $\R^d$ as the tensor $\R^{T^2} \otimes \R^{d'}$. Furthermore, we write $\be_i$ as a shorthand for the $i$-th vector in the standard basis of $\R^{d'}$.

The algorithm $\bB$ with an oracle $\cO$ to sample from $\tcP_{\bq}$ where $\bq$ is unknown to $\bB$ works as follows:
\begin{itemize}
\item Randomly draw $\bq_1, \dots, \bq_{T^2}$ i.i.d. from $\cU_{[0, 1]}^{\otimes d}$, and randomly sample $i^* \in [T^2]$.
\item Draw $n$ samples $(\bx_1, y_1), \dots, (\bx_n, y_n)$ independently as follows:
\begin{itemize}
\item Draw $i \sim [T^2]$.
\item If $i \ne i^*$, then draw $(\bx, y) \sim \tcP_{\bq_i}$ and let the sample be $(\bx \otimes \be_i, y)$.
\item If $i = i^*$, the draw $(\bx, y) \sim \tcP_{\bq}$ using $\cO$ and let the sample be $(\bx \otimes \be_i, y)$.
\end{itemize}
\item Run $\A$ on $(\bx_1, y_1), \dots, (\bx_n, y_n)$. Suppose that the output halfspace is $\bw$.
\item Write $\bw$ as $\sum_{i \in [T^2]} \bw^i \otimes \be_i$ for $\bw_1, \dots, \bw_{T^2} \in \R^{d'}$. Then, output $T \cdot \bw^{i^*}$.
\end{itemize}
Clearly, $\bB$ is $(1, o(1/n))$-DP and it takes $n / T = o(\sqrt{d'})$ samples in expectation from $\cP_{\bq}$.

For the ease of presentation, we will write $\cQ$ as a shorthand for the mixture of distribution where we draw $i \sim [T]$, and return $(\bx \otimes \be_i, y)$ where $(\bx, y) \sim \tcP_{\bq_i}$. Moreover, we write $\tcQ$ as a similar distribution but when $\bq_{i^*}$ is replaced by $\bq$. Under this notation, we have that
\begin{align*}
\E_{\bq \sim \cU_{[0, 1]}^{\otimes d}, \bw \sim \bB_{\tcP_{\bq}}}[\|\bw\|^2]
&= \E_{\bq_1, \dots, \bq_{T^2}, \bq \sim \cU_{[0, 1]}^{\otimes d}, i^* \sim [T], \bw \sim \A(\tcQ^n)}\left[\|T \cdot\bw^{i^*}\|^2\right] \\
&= \E_{\bq_1, \dots, \bq_{T^2} \sim \cU_{[0, 1]}^{\otimes d}, i^* \sim [T], \bw \sim \A(\cQ^n)}\left[\|T \cdot\bw^{i^*}\|^2\right] \\
&= \E_{\bq_1, \dots, \bq_{T^2} \sim \cU_{[0, 1]}^{\otimes d}, \bw \sim \A(\cQ^n)}\left[\frac{1}{T^2} \cdot \sum_{i^* \in [T^2]} \|T \cdot \bw^{i^*}\|^2\right] \\
&= \E_{\bq_1, \dots, \bq_{T^2} \sim \cU_{[0, 1]}^{\otimes d}, \bw \sim \A(\cQ^n)}\left[\|\bw\|^2\right] \\
&\leq 1.
\end{align*}

Finally, we will argue the accuracy of $\bB$ where $\gamma_0 = 0.01, \beta_0 = 2 \cdot 10^{-6}$. Once again we rewrite it as
\begin{align*}
\E_{\bq \sim \cU_{[0, 1]}^{\otimes d}, \bw \sim \bB_{\tcP_{\bq}}}[\err^{\tcP_\bq}_{\gamma_0}(\bw)]
&=  \E_{\bq_1, \dots, \bq_{T^2}, \bq, i^*, \bw \sim \A(\tcQ^n)}\left[\err^{\tcP_{\bq}}_{\gamma_0}(T \cdot \bw^{i^*})\right] \\
&= \E_{\bq_1, \dots, \bq_{T^2}, i^*, \bw \sim \A(\cQ^n)}\left[\err^{\tcP_{\bq}}_{\gamma_0 / T}(\bw^{i^*})\right] \\
(\text{Since } \gamma_0 / T \leq 0.1\gamma) &\leq \E_{\bq_1, \dots, \bq_{T^2}, i^*, \bw \sim \A(\cQ^n)}\left[\err^{\cP_{\bq}}_{0.1\gamma}(\bw^{i^*})\right] \\
&= \E_{\bq_1, \dots, \bq_{T^2}, \bw \sim \A(\cQ^n)}\left[\frac{1}{T^2} \sum_{i^* \in [T^2]} \err^{\tcP_{\bq}}_{0.1\gamma}(\bw^{i^*})\right] \\
&= \E_{\bq_1, \dots, \bq_{T^2}, \bw \sim \A(\cQ^n)}\left[\err^{\cQ}_{0.1\gamma}(\bw)\right].
\end{align*}

Now, notice that the halfspace $\bw^* = \frac{1}{T} \sum_{i \in [T^2]} \bq'_i \otimes \be_i$ (whose Euclidean norm is at most one) correctly classifies each point in $\supp(\cQ)$ with margin $\frac{0.01}{T} \geq \gamma$. As a result, the accuracy guarantee of $\A$ ensures that $\E_{\bw \sim \A(\cQ^n)}[\err^{\cQ}_{0.1\gamma}(\bw)] \leq \alpha(1 - \xi) + \xi \leq \beta_0$. Plugging into the above, we have
\begin{align*}
\E_{\bq \sim \cU_{[0, 1]}^{\otimes d}, \bw \sim \bB_{\tcP_{\bq}}}[\err^{\tcP_\bq}_{\gamma_0}(\bw)] \leq \beta_0,
\end{align*}
which contradicts with Lemma~\ref{lem:distributional-constant-gamma-hardness-convenient}.
\end{proof}
\section{Lower Bound for Non-Robust Learning of Halfspaces}

In this section, we provide a lower bound of $\Omega\left(\frac{1}{\eps\gamma^2}\right)$ on the sample complexity of \emph{non-robust} learners (Theorem~\ref{thm:lower-bound-non-robust}). While quantitatively similar, our lower bound significantly strengthens that of~\cite{NguyenUZ20} in two aspects: (1) our lower bounds hold against even \emph{improper} learners whereas the lower bound in~\citep{NguyenUZ20} is only valid against \emph{proper} learners and (2) our lower bound holds even against $(\eps, \delta)$-DP algorithms whereas that of~\cite{NguyenUZ20} is only valid when $\delta = 0$.

\lbnr*

To prove the above, we will require the following simple lemma, which states that the task of outputting an input bit requires $\Omega(1/\eps)$ equal samples in order to gain any non-trivial advantage over random guessing. The proof follows a straightforward packing argument.

\begin{lemma} \label{lem:lower-bound-non-robust-one-dim}
For $s \in \{\pm 1\}$, let $O_s$ denote the distribution which is $s$ with probability 1.
For any $\eps > 0$, there exists $\delta = \Omega(1/\eps)$ such that the following holds: There is no $(\eps, \delta)$-DP algorithm that can take at most $10^{-5} / \eps$ samples in expectation from $O_s$ for a random $s \in \{\pm 1\}$ and output $s$ correctly with probability 0.51.
\end{lemma}

\begin{proof}
We may assume that $\eps < 1$ as it is clear that the algorithm needs at least one sample to output $s$ correctly with probability 0.51. Furthermore, let $\delta = \frac{0.001}{1 - e^{-\eps}}$.

Suppose for the sake of contradiction that there is an algorithm $\A$ that takes in at most $10^{-5} / \eps$ samples in expectation and output $s$ correctly with probability 0.51. By Markov inequality, with probability 0.999, $\A$ takes at most $n := \lfloor 0.01 / \eps \rfloor$ samples. Let $\bB$ be the modification of $\A$ where $\bB$ draws $n$ samples and runs $\A$ on them but fails whenever $\A$ attempts to draw more than $n$ samples. We have that $\bB$ outputs $s$ correctly with probability $0.509$. In other words, we have
\begin{align} \label{eq:correctness-one-dim}
\Pr[\bB(s^n) = s] \geq 0.509,
\end{align}
where $s^n$ denote $n$ inputs all equal to $s$.

Since $\A$ is $(\eps, \delta)$-DP, $\bB$ is also $(\eps, \delta)$-DP. Suppose without loss of generality that $\Pr[\bB(\emptyset) \ne 1] \geq \Pr[\bB(\emptyset) \ne -1]$. This implies that $\Pr[\bB(\emptyset) \ne 1] \geq 0.5$. From $(\eps, \delta)$-DP of $\bB$, we have
\begin{align*}
\Pr[\bB(1^n) \ne 1] &\geq e^{-\eps} \Pr[\bB(1^{n-1}) \ne 1] - \delta \\
&\qquad \vdots \\
&\geq e^{-n \eps} \Pr[\bB(\emptyset) \ne 1] - \delta(1 + e^{-\eps} + \cdots + e^{-n\eps}) \\
&\geq e^{-0.01} \cdot 0.5 - 0.001 \\
&> 0.491
\end{align*}
which contradicts~\eqref{eq:correctness-one-dim}. This concludes our proof.
\end{proof}

We can now prove Theorem~\ref{thm:lower-bound-non-robust}. Roughly speaking, we ``embed'' the hard problem in Lemma~\ref{lem:lower-bound-non-robust-one-dim} into each of the $d = 1/\gamma^2$ dimensions, which results in the $d \cdot \Omega(1/\eps) = \Omega\left(\frac{1}{\eps\gamma^2}\right)$ lower bound.

\begin{proof}[Proof of Theorem~\ref{thm:lower-bound-non-robust}]
We prove this statement for any $\gamma < 1, \alpha \leq 0.4$ and $\xi \leq 0.0001$.

Let $\delta$ be the same as in Lemma~\ref{lem:lower-bound-non-robust-one-dim}, and let $d = \lfloor 1/\gamma^2 \rfloor$. Suppose for the sake of contradiction that there exists an $(\eps, \delta)$-DP $(\gamma, 0)$-robust learner $\A$ that takes in at most $n := \lfloor 10^{-5} d / \eps \rfloor$ samples and outputs a hypothesis with error at most $\alpha \leq 0.4$ with probability $1 - \xi \geq 0.9999$. We will use $\A$ to construct an algorithm $\bB$ that can solve the problem in Lemma~\ref{lem:lower-bound-non-robust-one-dim}.

 For every $i \in [d]$ and $s \in \{\pm 1\}$, we use $\cD_{i, s}$ to denote the uniform distribution on $(\be_i, s)$ and $(-\be_i, -s)$. Furthermore, for $\bs \in \{\pm 1\}^d$, we use $\cD_{\bs}$ to denote the mixture $\frac{1}{d} \sum_{i \in [d]} \cD_{i, s_i}$. Our algorithm $\B$ works as follows:
 \begin{itemize}
 \item Randomly sample $\bs \in \{\pm 1\}^d$ and randomly sample $i^* \in [d]$.
 \item Draw $n$ samples $(\bx_1, y_1), \dots, (\bx_n, y_n)$ independently as follows:
 \begin{itemize}
 \item Randomly pick $i \in [d]$.
 \item If $i \ne i^*$, then return a sample drawn from $\cD_{i, s_i}$.
 \item Otherwise, if $i = i^*$, sample $a \sim O_s$. Then return the sample $(\be_i, a)$ with probability 0.5; otherwise, return the sample $(-\be_i, -a)$.
 \end{itemize}
 \item Run $\A$ on $(\bx_1, y_1), \dots, (\bx_n, y_n)$ to get a hypothesis $h$.
 \item With probability 0.5, return $h(\be_{i^*})$. Otherwise, return $-h(-\be_{i^*})$.
 \end{itemize}
 It is obvious to see that $\bB$ is $(\eps, \delta)$-DP and that the expected number of samples $\bB$ draws from $O_s$ is $n / d \leq 10^{-5} / \eps$. Hence, we only need to show that $\bB$ outputs a correct answer with probability 0.51 to get a contradiction with Lemma~\ref{lem:lower-bound-non-robust-one-dim}.

 Since $s$ is uniformly draw from $\{\pm 1\}$, the probability that $\bB$ outputs the \emph{incorrect} answer is equal to
 \begin{align*}
 &\E_{\bs \sim \{\pm 1\}^d, i \in [d], h \sim \A(\cD_{\bs}^{\otimes n})}\left[\frac{1}{2}\1\left[h(\be_i) \ne s_i\right] + \frac{1}{2}\1\left[h(-\be_i) \ne -s_i\right]\right] \\
 &= \E_{\bs \sim \{\pm 1\}^d, h \sim \A(\cD_{\bs}^{\otimes n})}\left[\frac{1}{d} \sum_{i \in [d]} \left(\frac{1}{2}\1\left[h(\be_i) \ne s_i\right] + \frac{1}{2}\1\left[h(-\be_i) \ne -s_i\right]\right)\right] \\
 &= \E_{\bs \sim \{\pm 1\}^d, h \sim \A(\cD_{\bs}^{\otimes n})}\left[\err^{\cD_{\bs}}_0(h)\right].
 \end{align*}
 Now, notice that any $(\bx, y) \in \supp(\cD_{\bs})$ is correctly classified by the halfspace $\bz := \frac{1}{\sqrt{d}} \sum_{i \in [d]} \be_i$ with margin $1/\sqrt{d} \geq \gamma$. As a result, the accuracy guarantee of $\A$ ensures that $\E_{h \sim \A(\cD_{\bs}^{\otimes n})}[\err^{\cD_{\bs}}_0(h))] \leq 1 \cdot 0.0001 + 0.4 \cdot 0.9999 < 0.41$. Thus, we can conclude that $\bB$ outputs the \emph{correct} answer with probability at least $1 - 0.41 > 0.59$. This contradicts Lemma~\ref{lem:lower-bound-non-robust-one-dim}.
\end{proof}

\section{Pure DP Robust Learner}

In this section, we give a pure-DP algorithm for robust learning of halfspaces:

\expmech*


To prove this result, we will also need the following generalization bound due to~\cite{bartlett2002rademacher}:

\newcommand{\hgamma}{\hat{\gamma}}
\newcommand{\hxi}{\hat{\xi}}

\begin{lemma}[Generalization Bound for Large Margin Halfspaces~\citep{bartlett2002rademacher}] \label{lem:generalization}
Suppose $\hgamma, \hxi \in [0, 1]$ and let $\cD$ be any distribution on $\bB^d \times \{\pm 1\}$. If we let $\bX$ be drawn from $\cD^{\otimes n}$, then the following holds with probability $1 - \hxi$:
\begin{align*}
\forall \bw \in \bB^d, &\err^{\cD}_{0.95\hgamma}(\bw) \leq \err^{\bX}_{\hgamma}(\bw) + 400\sqrt{\frac{\ln(4/\hxi)}{n\hgamma^2}}.
\end{align*}
\end{lemma}

\begin{proof}[Proof of Theorem~\ref{thm:exp-mech}]
We will prove this for $\xi = 0.9$. Let $\Lambda = 10^6 \cdot \sqrt{\log(1/\alpha)} \cdot \max\{\sqrt{d}, 1/\gamma\}$, and $n = \frac{10^4 \Lambda^2}{\eps \alpha} + \frac{10^{10}}{\alpha^2 \gamma^2} = O\left(\frac{\log(1/\alpha)}{\alpha \epsilon} \cdot \max\{d, 1 / \gamma^2\} + \frac{1}{\alpha^2 \gamma^2}\right)$. 

Our algorithm samples $(\bx_1, y_1), \dots, (\bx_n, y_n)$ from $\cD$, and then employs the exponential mechanism of~\cite{McSherryT07}. Specifically, let $\mu$ be the density of the uniform measure over the unit sphere in $\R^d$. Then, on the input dataset $\bX = ((\bx_1, y_1), \dots, (\bx_n, y_n))$, we define the scoring function $q$ by
\begin{align*}
q(\bX, \bw) = -n \cdot \err^{\bX}_{0.95\gamma}(\bw).
\end{align*}
Then, we output $\bhw$ drawn from the distribution with density $\mu'(\bw) \propto \mu(\bw) \cdot \exp\left(\frac{\eps}{2} \cdot q(\bX, \bw)\right)$.

We will next argue the accuracy guarantee of the algorithm. Similar to~\citep{McSherryT07}, let $S_t := \{\bw \mid q(\bX, \bw) \geq - t\}$. We start by showing that, with probability 0.99, we have $\bhw \in S_{0.5 \alpha n}$. To prove this, we will use the following result from~\citep{McSherryT07}:
\begin{lemma} \label{lem:exp-mech-accuracy}
For any $t \geq 0$, $\Pr[\bhw \notin S_{2t}] \leq \exp(-\epsilon t / 2) / \mu(S_t)$.
\end{lemma}

In light of Lemma~\ref{lem:exp-mech-accuracy}, it suffices for us to provide a lower bound for $\mu(S_{0.25\alpha n})$. Recall from the realizable assumption that, there exists a unit-norm $\bw^*$ such that $\err^{\bX}_{\gamma}(\bw^*) = 0$. Since $\mu(S_{0.25\alpha n})$ is rotational-invariant, we may assume for notational convenience that $\bw^* = \be_d$, the $d$-th vector in the standard basis. In this notation, a sample $\bw \sim \mu$ may be obtained by:
\begin{itemize}
\item Sample $w_d \sim \cN(0, 1)$,
\item Sample $\bw_{\perp} \sim \cN(0, I_{(d-1) \times (d-1)})$,
\item Let $\bw = \frac{1}{T} \left(\bw_{\perp} \circ w_d\right)$ where $T = \sqrt{\|\bw_{\perp}\|^2 + w_d^2}$.
\end{itemize}

Fix $i \in [n]$. We will now bound the probability $\Pr[y_i \left<\bw, \bx_i\right> \leq 0.95\gamma \mid w_d \geq \Lambda]$. Let us write $y_i\bx_i$ as $\bx_{\perp} \circ x_d$. 
 $\left<\bw_{\perp}, \bx_{\perp}\right>$ is distributed as $\cN(0, \|\bx_{\perp}\|)$. Since $\|\bx_{\perp}\| \leq 1$, we may apply standard tail bound of Gaussian which gives 
\begin{align}
\Pr[\left<\bw_{\perp}, \bx_{\perp}\right> < -0.01\Lambda / \gamma] \leq \Pr[\left<\bw_{\perp}, \bx_{\perp}\right> < 10^4\sqrt{\log(1/\alpha)}] \leq 0.1\alpha.
\end{align}
Observe also that $\|\bw_{\perp}\|^2$ is simply distributed as $\chi^2_{d - 1}$. Hence, via standard tail bound (e.g.,~\citep{laurent2000adaptive}), we have
\begin{align}
\Pr[\|\bw_{\perp}\| > 0.01 \Lambda] \leq \Pr[\|\bw_{\perp}\| > 10^4 \sqrt{d \log(1/\alpha)}] \leq 0.1\alpha.
\end{align}
Furthermore, notice that when $w_d \geq \Lambda, \left<\bw_{\perp}, \bx_{\perp}\right> \geq -0.01\Lambda / \gamma$ and $\|\bw_{\perp}\| \leq 0.01 \Lambda$, we have $y_i \left<\bw, \bx_i\right> > 0.95\gamma$. As a result, a union bound and the independence of $w_d$ and $\bw_{\perp}$ implies that
\begin{align}
\Pr[y_i \left<\bw, \bx_i\right> \leq 0.95\gamma \mid w_d \geq \Lambda] &\leq \Pr[\left<\bw_{\perp}, \bx_{\perp}\right> < -0.01\Lambda / \gamma] + \Pr[\|\bw_{\perp}\| > 0.01 \Lambda] \nonumber \\
&\leq 0.2 \alpha. \label{eq:cond-each-sample}
\end{align}

From~\eqref{eq:cond-each-sample} and from the linearity of the expectation, we have that
\begin{align*}
\E[|\{i \in [n] \mid y_i \left<\bw, \bx_i\right> \leq 0.95\gamma\}| \mid w_d \geq \Lambda] \leq 0.2\alpha n.
\end{align*}
By Markov's inequality, we may conclude that
\begin{align*}
\Pr[\bw \in S_{0.25\alpha n} \mid w_d \geq \Lambda] \geq 0.1.
\end{align*}
Finally, recall that $w_d$ is distributed as $\cN(0, 1)$, which implies that $\Pr[w_d \geq \Lambda] \geq 2^{-10\Lambda^2}$. This gives
\begin{align} \label{eq:mu-bound}
\mu(S_{0.25\alpha n}) = \Pr[\bw \in S_{0.25\alpha n} \mid w_d \geq \Lambda] \Pr[w_d \geq \Lambda] \geq 0.1 \cdot 2^{-10\Lambda^2} \geq 2^{-20\Lambda^2}.
\end{align}

Hence, applying Lemma~\ref{lem:exp-mech-accuracy}, we get that
\begin{align*}
\Pr[\bhw \notin S_{0.5\alpha n}] &\leq \frac{\exp(-\eps t / 2)}{\mu(S_t)} \\
&\overset{\eqref{eq:mu-bound}}{\leq} \frac{\exp(-0.125 \eps \alpha n)}{2^{-20\Lambda^2}} \\
(\text{From our choice of } n) &\leq 0.99.
\end{align*}
In other words, with probability 0.99, we have $\err^{\bX}_{0.95\gamma}(\bhw) \leq 0.5\alpha$. Finally, via the generalization bound (Lemma~\ref{lem:generalization} with $\hgamma = 0.95\gamma$), we also have $\err_{0.9\gamma}^{\cD}(\bhw) \leq \alpha$ with probability 0.9 as desired.
\end{proof}
\section{Approximate-DP Robust Learner}

In this section, we describe our approximate-DP learner and prove its guarantee, restated below:

\perceptron*

As alluded to earlier, this algorithm is a noised and batch version of the margin perceptron algorithm~\citep{DudaH73,CollobertB04}. The algorithm is presented in Algorithm~\ref{alg:dp-batch-perceptron}.

The rest of this section is organized as follows. In the next subsection, we provide the utility analysis of the algorithm. Then, in Subsection~\ref{subsec:perceptron-privacy}, we analyze its privacy guarantee. Finally, we set the parameters and prove Theorem~\ref{thm:perceptron-apx-DP-simplified} in Section~\ref{subsec:perceptron-finishing}.

\subsection{Utility Analysis}
\label{subsec:perceptron-utility}

\newcommand{\boundparallel}{\left(0.01 \alpha \gammagap B\right)}

Suppose that there exists $\bw^* \in \bB^d$ with $\err_{\gamma}(\bw^*) = 0$. Furthermore, let $\gamma' = 0.95\gamma$, $\gammagap := \gamma - \gamma'$ and $B := pn$.
Throughout the analysis, we will assume that the following ``good'' events occur:
\begin{itemize}
\item $E_{\text{batch-size}}$: For all $i \in [T]$, $|S_i| \leq 1.5B$.
\item $E_{\text{noise-norm}}$: For all $i \in [T]$, $\|\bg_i\| \leq B\sqrt{\alpha}$.
\item $E_{\text{parallel}}$: For all $i \in [T]$, $\left<\bw_{i - 1} + \bu_i, \bg_i\right> \leq 0.01 \alpha \gammagap B \cdot \|\bw_{i - 1} + \bu_i\|$.
\item $E_{\text{opt-noise}}$: For all $i \in [T]$, $\left<\bw^*, \bg_i\right> \geq -0.01 \alpha \gammagap B$.
\item $E_{\text{mistake-noise}}$: For all $i \in [T]$, $\nu_i \in [-0.1\alpha B, 0.1\alpha B]$.
\item $E_{\text{sampled-mistake}}$: For all $i \in [T]$ such that\footnote{Similar to before, we use $\bX$ to denote $(\bx_1, y_1), \dots, (\bx_n, y_n)$.} $\err^{\bX}_{\gamma'}\left(\frac{\bw_{i - 1}}{\|\bw_{i - 1}\|}\right) > 0.5\alpha$, we have $|M_i| \geq 0.4\alpha B$.
\end{itemize}
Later on, we will select the parameters $p, n, T, b, \sigma$ so that these events happen with high probability.

\begin{lemma} \label{lem:perceptron-util}
Let $T = \lceil \frac{1500}{\alpha \gammagap^2} \rceil$. If the events $E_{\text{batch-size}}, E_{\text{noise-norm}}, E_{\text{parallel}}, E_{\text{opt-noise}}, E_{\text{mistake-noise}}$ and $E_{\text{sampled-mistake}}$ all occur, then the algorithm outputs $\bw$ such that $\err^{\bX}_{\gamma'}(\bw) \leq 0.5\alpha$.
\end{lemma}

\begin{proof}
We will show that we always execute Line~\ref{step:return}. Once this is the case, $E_{\text{mistake-noise}}$ and $E_{\text{sampled-mistake}}$ imply that the output $\bw_i / \|\bw_i\|$ satisfies $\err^{\bX}_{\gamma'}(\bw_i / \|\bw_i\|) \leq 0.5\alpha$ as desired. 

To prove that we execute Line~\ref{step:return}, let us assume for the sake of contradiction that this is not the case, i.e., that the algorithm continues until reaching the end of the $T$-th iteration.

From our assumption that $E_{\text{mistake-noise}}$ occurs and from the fact that Line~\ref{step:return} was not executed, we have $|M_i| \geq 0.2 \alpha B$ for all $i \in [T]$. Let $m_i := \sum_{j \in [i]} |M_i|$ denote the number of $\gamma'$-margin mistakes seen up until the end of the $i$-th iteration; from the previous bound on $M_i$, we have
\begin{align} \label{eq:mistake-lb}
m_i \geq 0.2\alpha B i.
\end{align}

Now, notice that
\begin{align}
\left<\bw^*, \bw_T\right> 
&= \left<\bw^*, \left(\sum_{i \in [T]} \sum_{(\bx, y) \in M_i} y \cdot \bx \right) + \sum_{i \in [T]} \bg_i\right> \nonumber \\
&= \left(\sum_{i \in [T]} \sum_{(\bx, y) \in M_i} y \cdot \left<\bw^*, \bx\right> \right) + \sum_{i \in [T]} \left<\bw^*, \bg_i\right> \nonumber \\
(\text{From } \err^{\bX}_{\gamma'}(\bw^*) = 0 \text{ and } E_{\text{opt-noise}}) &\geq \left(\sum_{i \in [T]} \sum_{(\bx, y) \in M_i} \gamma \right) + \sum_{i \in [T]} -0.01\alpha \gammagap B \nonumber \\
&= m_T \gamma -0.01\alpha \gammagap B T \nonumber \\
&\overset{\eqref{eq:mistake-lb}}{\geq} m_T (\gamma - 0.05 \gammagap). \label{eq:dot-prod-bound}
\end{align}

Furthermore, for every $i \in [T]$, we have that
\begin{align}
\|\bw_i\|^2 &= \|\bw_{i - 1} + \bu_i + \bg_i\|^2 \nonumber \\
&= \|\bw_{i - 1} + \bu_i\|^2 + 2\left<\bw_{i - 1} + \bu_i, \bg_i\right> + \|\bg_i\|^2 \nonumber \\
(\text{From } E_{\text{parallel}}) &\leq \|\bw_{i - 1} + \bu_i\|^2 + 0.02 \alpha \gammagap B \cdot \|\bw_{i - 1} + \bu_i\| + \|\bg_i\|^2 \nonumber \\
(\text{From } E_{\text{noise-norm}}) &\leq \|\bw_{i - 1} + \bu_i\|^2 + 0.02 \alpha \gammagap B \cdot \|\bw_{i - 1} + \bu_i\| + \alpha B^2 \nonumber \\
&\leq \|\bw_{i - 1}\|^2 + 2\left<\bw_{i - 1}, \bu_i\right> + \|\bu_i\|^2 + 0.02 \alpha \gammagap B \cdot \left(\|\bw_{i - 1}\| + \|\bu_i\|\right) + \alpha B^2 \label{eq:norm-intermediate}
\end{align}

We can bound $\left<\bw_{i - 1}, \bu_i\right>$ as follows:
\begin{align*}
\left<\bw_{i - 1}, \bu_i\right> = \sum_{(\bx, y) \in M_i} y \cdot \left<\bw_{i - 1}, \bx\right> \leq |M_i| \cdot \gamma' \|\bw_{i - 1}\|,
\end{align*}
where the inequality follows from the condition on Line~\ref{step:margin-check}.

Furthermore, we also have that
\begin{align*}
\|\bu_i\| = \left\|\sum_{(\bx, y) \in M_i} y \cdot \bx\right\| \leq \sum_{(\bx, y) \in M_i} \|\bx\| \leq |M_i|.
\end{align*}

Plugging the above two inequalities into~\eqref{eq:norm-intermediate}, we get
\begin{align*}
\|\bw_i\|^2 &\leq \|\bw_{i - 1}\|^2 + \left(2|M_i| \gamma' + 0.02 \alpha \gammagap B\right) \cdot \|\bw_{i - 1}\| + (|M_i|^2 + 0.02 \alpha \gammagap B |M_i| + \alpha B^2) \\
&\leq \|\bw_{i - 1}\|^2 + \left(2|M_i| \gamma' + 0.02 \alpha \gammagap B\right) \cdot \|\bw_{i - 1}\| + (|M_i|^2 + 0.02B \cdot |M_i| + \alpha B^2) \\
&\leq \|\bw_{i - 1}\|^2 + \left(2|M_i| \gamma' + 0.02 \alpha \gammagap B\right) \cdot \|\bw_{i - 1}\| + 2B|M_i| + \alpha B^2,
\end{align*}
where in the last inequality we use the fact that $|M_i| \leq 1.5B$ which follows from $E_{\text{batch-size}}$.

The above inequality implies that
\begin{align*}
\|\bw_i\| \leq \|\bw_{i - 1}\| + \left(|M_i| \gamma' + 0.01 \alpha \gammagap B\right) + \frac{B|M_i| + 0.5 \alpha B^2}{\|\bw_{i - 1}\|}.
\end{align*}
Notice that when $\|\bw_{i - 1}\| \geq \frac{100 B}{\gammagap}$, we have that
\begin{align*}
\|\bw_i\| &\leq \|\bw_{i - 1}\| + \left(|M_i| \gamma' + 0.01 \alpha \gammagap B\right) + 0.01 |M_i| \gammagap + 0.01 \alpha \gammagap B \\
&= \|\bw_{i - 1}\| + 0.02 \alpha \gammagap B + \left(\gamma' + 0.01 \gammagap\right) \cdot |M_i|.
\end{align*}
As a result, we get\footnote{Note that the first term $\frac{200 B}{\gammagap}$ comes from an observation that if $i_0$ is the smallest index for which $\|\bw_{i_0}\| \geq \frac{100 B}{\gammagap}$, then~\eqref{eq:norm-intermediate} implies that $\|\bw_{i_0}\|_0 \leq \frac{200B}{\gammagap}$.}
\begin{align}
\|\bw_T\| 
&\leq \frac{200 B}{\gammagap} + 0.02 \alpha \gammagap B T + \left(\gamma' + 0.01 \gammagap\right) \cdot \left(\sum_{i \in [T]} |M_i|\right) \nonumber \\
&= \frac{200 B}{\gammagap} + 0.02 \alpha \gammagap B T + \left(\gamma' + 0.01 \gammagap\right) \cdot m_T \nonumber \\
&\overset{\eqref{eq:mistake-lb}}{\leq} \frac{200 B}{\gammagap} + \left(\gamma' + 0.11 \gammagap\right) \cdot m_T. \label{eq:norm-final}
\end{align}

From~\eqref{eq:dot-prod-bound} and~\eqref{eq:norm-final}, we have
\begin{align*}
m_T (\gamma - 0.05 \gammagap) \leq \frac{200 B}{\gammagap} + m_T\left(\gamma' + 0.11 \gammagap\right),
\end{align*}
which implies that
\begin{align*}
m_T &\leq \frac{200 B}{\gammagap\left(\gamma - \gamma' - 0.16\gammagap\right)} \\
&< \frac{200 B}{0.8 \gammagap^2} \\
&= 250 B / \gammagap^2,
\end{align*}
which contradicts~\eqref{eq:mistake-lb} and our choice of $T = \lceil \frac{1500}{\alpha \gammagap^2} \rceil$.
\end{proof}

\subsection{Privacy Analysis}
\label{subsec:perceptron-privacy}

\begin{lemma} \label{lem:perceptron-dp}
For any $\eps, \delta \in (0, 1)$ and any $T \in \N$, let $p = \frac{1}{\sqrt{T}}, \sigma = \frac{100\ln(T/\delta)}{\eps}$ and $b = \frac{100\sqrt{\ln(T/\delta)}}{\eps}$. Then, Algorithm~\ref{alg:dp-batch-perceptron} is $(\eps, \delta)$-DP.
\end{lemma}

To prove this, we require the following the results on amplification by subsampling\footnote{Amplification by subsampling results are often stated with the new $\eps$ being $\ln(1 + p(e^{\eps} - 1))$ which is no more than $2p\eps$ (from Bernoulli's inequality and from $1 + x \leq e^x$ for all $x \in \R$).} and advanced composition.

\begin{lemma}[Amplification by Subsampling~\citep{BalleBG18}] \label{lem:subsampling-dp}
Let $\A$ be any $(\eps_0, \delta_0)$-DP algorithm such that $\eps_0, \delta_0 \in (0, 1)$. Let $\bB$ be an algorithm that independently selects each input sample w.p. $p$ and runs $\A$ on this subsampled input dataset. Then, $\bB$ is $(2p\eps_0, p\delta_0)$-DP.
\end{lemma}

\begin{lemma}[Advanced Composition~\citep{DworkRV10}] \label{lem:advanced-composition}
Suppose that $\bB$ is an algorithm resulting from running an $(\eps_0, \delta_0)$-DP algorithm $T$ times (possibly adaptively), where $\eps_0, \delta_0 \in (0, 1)$. Then, $\bB$ is $(\eps', (T + 1)\delta_0)$-DP where
\begin{align*}
\eps' = \sqrt{2T \ln(1/\delta_0)} \cdot \eps_0 + 2 T \eps_0^2.
\end{align*}
\end{lemma}

\begin{proof}[Proof of Lemma~\ref{lem:perceptron-dp}]
Let $\eps_0 = \frac{\eps}{20\sqrt{\ln(T/\delta)}}$ and $\delta_0 = \frac{\delta}{2\sqrt{T}}$. The Gaussian mechanism with noise standard deviation $\sigma$ is $(0.5\eps_0, \delta_0)$-DP~\cite[Appendix A]{dwork2014algorithmic} whereas the Laplace mechanism with parameter $b$ is $0.5\eps_0$-DP~\citep{DworkMNS06}\footnote{In both cases, the $\ell_2$ sensitivity and the $\ell_1$ sensitivity respectively are bounded by one. For the former, this is because each sample effects $\bw$ only by $y \cdot \bx$ and $\|y \cdot \bx\|_2 = \|\bx\| \leq 1$.}. As a result, without subsampling, each iteration is $(\eps_0, \delta_0)$-DP. With the subsampling, Lemma~\ref{lem:subsampling-dp} implies that each iteration is $(2p\eps_0, p\delta_0)$-DP. Finally, we may apply Lemma~\ref{lem:advanced-composition} ensures that the final algorithm is $(\eps', \delta')$-DP for
\begin{align*}
\eps' = \sqrt{2 T \ln(1 / (p\delta))} \cdot (2p\eps_0) + 2T(2p \eps_0)^2 \leq \eps,
\end{align*}
and
\begin{align*}
\delta' = (T + 1)p\delta_0 \leq 2\sqrt{T} \delta_0 = \delta,
\end{align*}
which concludes our proof.
\end{proof}

\subsection{Putting Things Together}
\label{subsec:perceptron-finishing}

\begin{proof}[Proof of Theorem~\ref{thm:perceptron-apx-DP-simplified}]
Let $T = \lceil \frac{1500}{\alpha \gammagap^2} \rceil$ be as in Lemma~\ref{lem:perceptron-util}, and let $p = \frac{1}{\sqrt{T}}, \sigma = \frac{100\ln(T/\delta)}{\eps}$ and $b = \frac{100\sqrt{\ln(T/\delta)}}{\eps}$ be as in Lemma~\ref{lem:perceptron-dp}. Finally, let $n = \lceil \frac{100 \sqrt{d} \sigma \log T}{p \sqrt{\alpha}} + \frac{1000\sigma\sqrt{\log T}}{p \alpha \gamma} + \frac{100 \log T}{\alpha} + \frac{10^{10}}{\alpha^2 \gamma^2} \rceil$. Notice that $n = O_{\alpha}\left(\frac{1}{\eps\gamma}\left(\sqrt{d} + \frac{1}{\gamma}\right) \cdot (\log T)^2\right)$ as claimed.

From Lemma~\ref{lem:perceptron-dp}, our algorithm with the above parameters is $(\eps, \delta)$-DP. Furthermore, the expected running time of the algorithm is $pnT = O(n\sqrt{T}) = O\left(\frac{n}{\gamma\sqrt{\alpha}}\right)$. Moreover, it can be verified via standard concentration inequalities that all of the events required in Lemma~\ref{lem:perceptron-util} happens w.p. 0.99, which means that we output a halfspace $\bw$ with $\err_{\gamma'}^{\bX}(\bw) \leq 0.5\alpha$. Finally, the generalization bound (Lemma~\ref{lem:generalization} with $\hgamma = \gamma'$) implies that $\err_{0.9\gamma}^{\cD}(\bw) \leq \alpha$ as desired.
\end{proof}
\section{Additional Experiments}\label{app:experiment_details}
\newcommand{\hsigma}{\hat{\sigma}}
\newcommand{\hphi}{\hat{\phi}}
\newcommand{\hd}{\hat{d}}
\newcommand{\bp}{\mathbf{p}}

\subsection{Adversarial Robustness Evaluation on USPS Dataset}

\begin{figure*}[ht]
\minipage{0.33\textwidth}
\includegraphics[width=\linewidth]{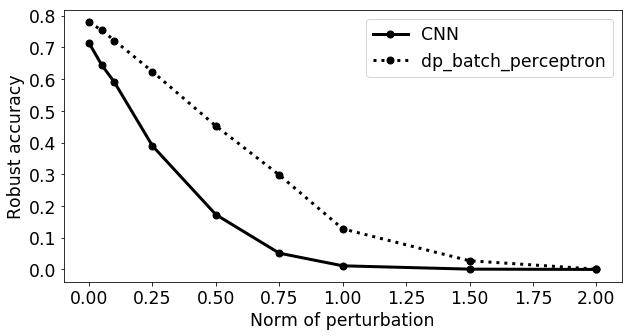}
\subcaption{$\eps = 0.5$}
\endminipage\hspace{-10pt}\hfill
\minipage{0.30\textwidth}
\includegraphics[width=1.1\linewidth]{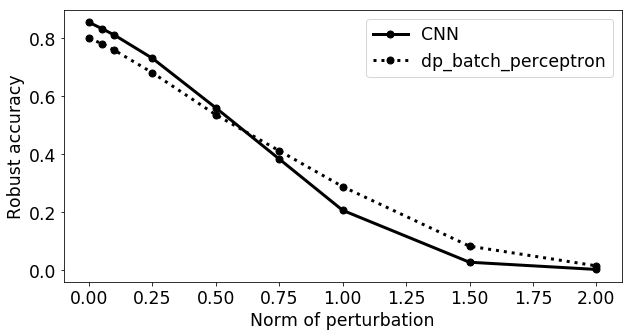}
\subcaption{$\eps = 1$}
\endminipage\hspace{+5pt}\hfill
\minipage{0.33\textwidth}
\includegraphics[width=\linewidth]{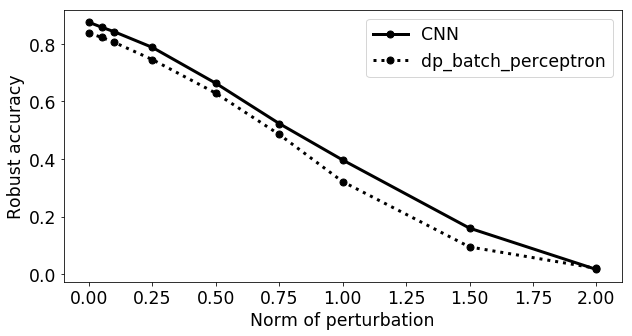}
\subcaption{$\eps = 2$}
\endminipage
\caption{Robustness accuracy comparison between DP-SGD-trained Convolutional neural networks and DP Batch Perceptron halfspace classifiers on USPS dataset for a fixed privacy budget. In all three plots, $\delta=10^{-5}$ but $\eps$ varies from 0.5, 1, and 2.}
\label{fig:robustness_usps}
\end{figure*}

We compare the robust accuracy of DP Batch Perceptron classifiers and DP-SGD-trained neural networks in Figure~\ref{fig:robustness_usps} for $\delta = 10^{-4}$ and $\eps = 0.5, 1, 2$.
The architecture and parameters follow the same setup described in Section~\ref{sec:exp}. In the case of $\eps = 0.5$, while both classifiers have similar test accuracies (without any perturbation, $\gamma=0$), as $\gamma$ increases, the robust accuracy rapidly degrades for the DP-SGD-trained neural network compared to that of the DP Batch Perceptron model. This overall trend persists for $\eps = 1$; the CNN starts off with larger test accuracy when $\gamma = 0$ but is eventually surpassed by the halfspace classifier as $\gamma$ increases. On the other hand, when $\eps = 2$, the CNN maintains slightly higher robust accuracy for most perturbation norms in consideration. 

\subsection{Experiments with Gaussian Kernel}

It is well-known that accuracy of linear classifiers for digit classifications can be significantly improved via kernel methods (see, e.g.,~\citep{Lecun98gradient-basedlearning,ScholkopfSBGNPV97}). Here we would like to privately train linear classifiers with Gaussian kernels. Recall that the Gaussian kernel is that of the form
\begin{align*}
k(\bx, \bx') = \exp\left(-\frac{\|\bx - \bx'\|^2}{2\hsigma^2}\right),
\end{align*}
where $\hsigma$ is the so-called \emph{width} parameter.

Unlike the standard (non-kernel) setting, it is unclear in the Gaussian kernel setting how the noise should be added to obtain DP; the kernel space themselves is not of finite dimension, and the classifier is typically only implicitly represented. To handle this, we follow the approach of \cite{RahimiR07} (also used in DP-SVM~\citep{rubinstein2009learning}). Specifically,~\cite{RahimiR07} shows that the following approximate embedding $\hphi: \R^d \to \bB^{2 \hd}$ has a property that $\left<\hphi(\bx), \hphi(\bx')\right>$ is close to $k(\bx, \bx')$:
\begin{align*}
\hphi(\bx) := \frac{1}{\sqrt{\hd}}\left(\cos(\left<\rho_1, \bx\right>, \dots, \cos(\left<\rho_{\hd}, \bx\right>, \sin(\left<\rho_1, \bx\right>), \dots, \sin(\left<\rho_{\hd}, \bx\right>)\right),
\end{align*}
where $\rho_1, \dots, \rho_{\hd}$ are i.i.d. sampled from $\cN(\bzero, \frac{1}{\hsigma^2} \cdot I_{d \times d})$. Below we write $\sigma^*$ to denote $1/\hsigma$.

To summarize, this approach allows us to train with (approximate) Gaussian kernel as follows (where $\sigma^*, \hd$ are hyperparameters):
\begin{enumerate}
\item Randomly sample $\rho_1, \dots, \rho_{\hd}$ i.i.d. from $\cN(\bzero, (\sigma^*)^2 \cdot I_{d \times d})$.
\item For each class $y$, use DP-Batch-Perceptron on $(\hphi(\bx_1), y_1), \dots, (\hphi(\bx_n), y_n)$ to train a halfspace $\bw^{(y)} \in \R^{2\hd}$ for the $y$-vs-rest classifier.
\item When we would like to predict $\bx \in \R^d$, compute $\argmax_{y \in \{1, \dots, 10\}} \left<\bw^{(y)}, \hphi(\bx)\right>$.
\end{enumerate}
Notice here that the DP guarantee (in the second step) is exactly the same as the DP-Batch-Perceptron guarantee for the non-kernel setting. Similar to Figure \ref{fig:mnist_usps}, we report the (non-robust) test accuracy of DP Batch Perceptron algorithm with Gaussian kernel included in Figure \ref{fig:mnist_usps_kernel}, across different $\eps$ values (first column) and different $\delta$ values (middle column). We find that kernel learning helps to boost performance overall, the gain in accuracy is particularly significant in the case of MNIST dataset (top row). 

\begin{figure}[ht]
\minipage{0.38\textwidth}
\includegraphics[width=\linewidth]{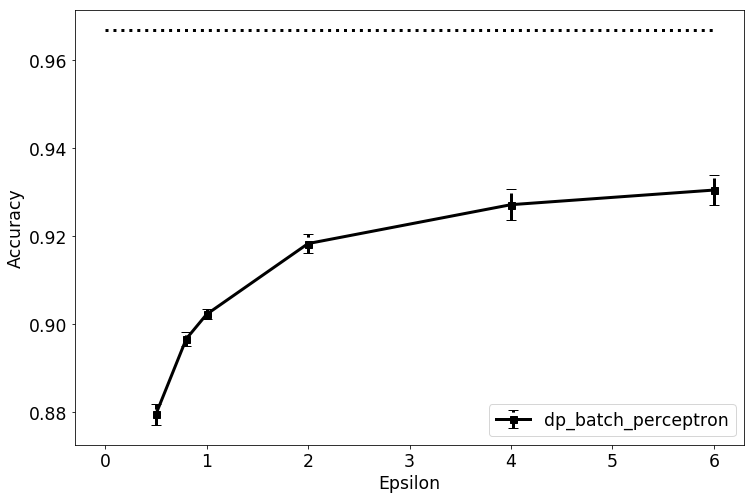}
\endminipage\hspace{-30pt}\hfill
\minipage{0.25\textwidth}
\includegraphics[width=1.1\linewidth]{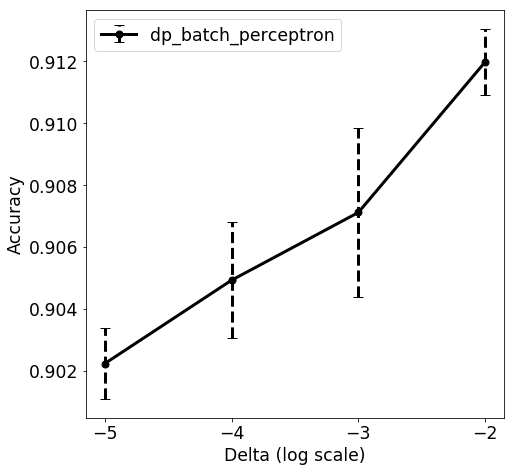}
\endminipage\hspace{-15pt}\hfill
\minipage{0.33\textwidth}
\includegraphics[width=\linewidth]{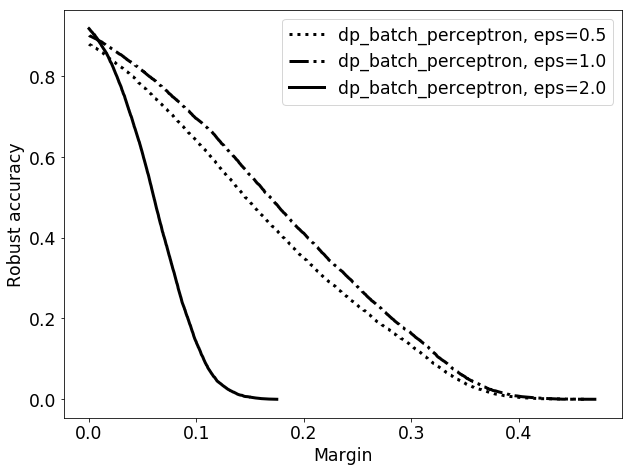}
\endminipage

\minipage{0.38\textwidth}
\includegraphics[width=\linewidth]{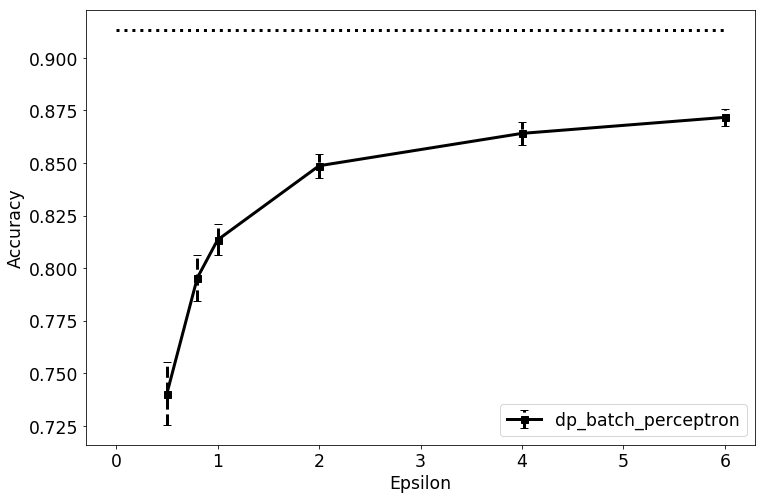}
\subcaption{Accuracy as $\eps$ varies}
\endminipage\hspace{-30pt}\hfill
\minipage{0.25\textwidth}
\includegraphics[width=1.1\linewidth]{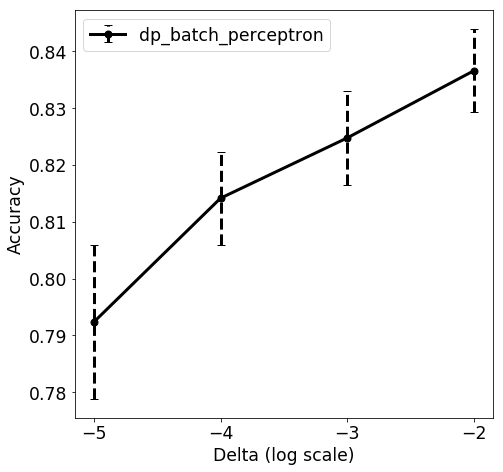}
\subcaption{Accuracy as $\delta$ varies}
\endminipage\hspace{-15pt}\hfill
\minipage{0.33\textwidth}
\includegraphics[width=\linewidth]{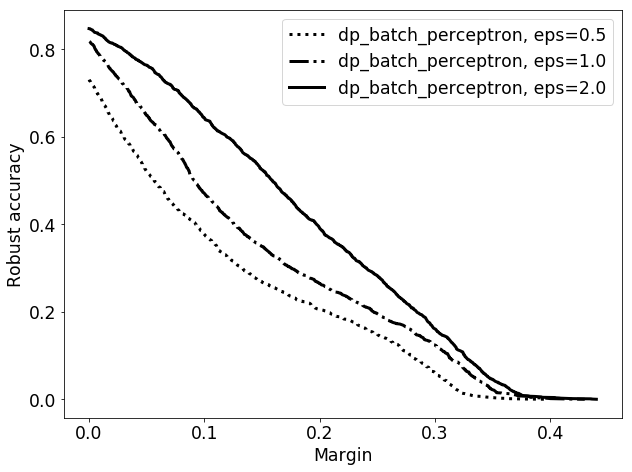}
\subcaption{Robust accuracy as $\eps$ varies}
\endminipage
\caption{Performance on the MNIST (top row) and USPS (bottom row) datasets with Gaussian kernel. The horizontal dotted line indicates performance when $\eps=\infty$ (no noise). The width of the kernel for both datasets is tuned as a hyperparameter with values 2, 3.5, 5, 7.5, 10.}
\label{fig:mnist_usps_kernel}
\end{figure}

\begin{figure*}[ht]
\minipage{0.33\textwidth}
\includegraphics[width=\linewidth]{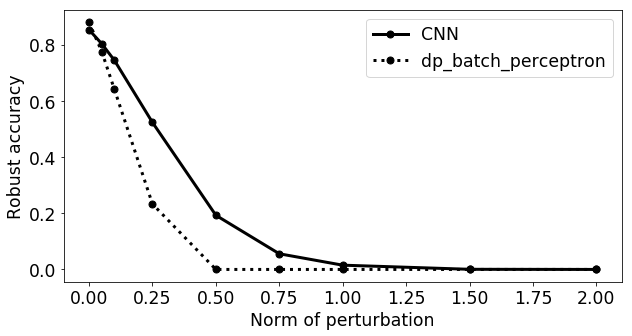}
\subcaption{$\eps = 0.5$}
\endminipage\hspace{-10pt}\hfill
\minipage{0.30\textwidth}
\includegraphics[width=1.1\linewidth]{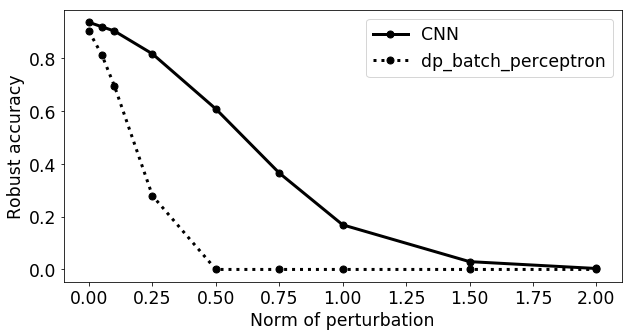}
\subcaption{$\eps = 1$}
\endminipage\hspace{+5pt}\hfill
\minipage{0.33\textwidth}
\includegraphics[width=\linewidth]{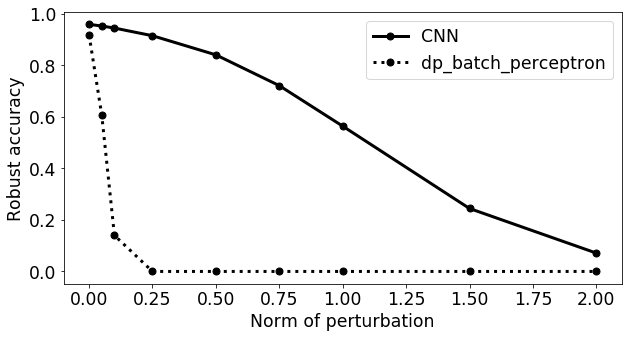}
\subcaption{$\eps = 2$}
\endminipage

\minipage{0.33\textwidth}
\includegraphics[width=\linewidth]{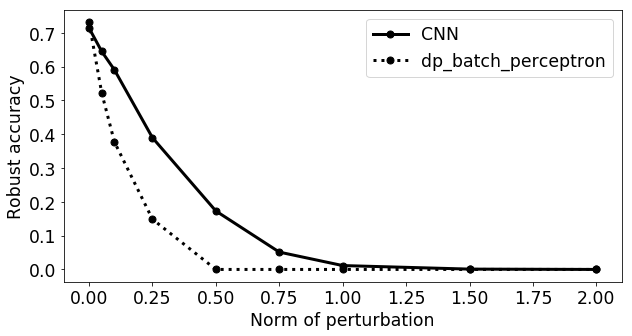}
\subcaption{$\eps = 0.5$}
\endminipage\hspace{-10pt}\hfill
\minipage{0.30\textwidth}
\includegraphics[width=1.1\linewidth]{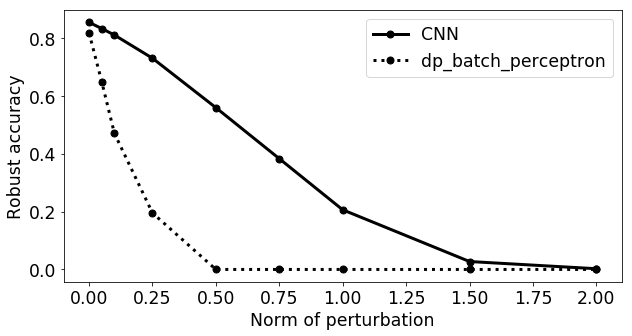}
\subcaption{$\eps = 1$}
\endminipage\hspace{+5pt}\hfill
\minipage{0.33\textwidth}
\includegraphics[width=\linewidth]{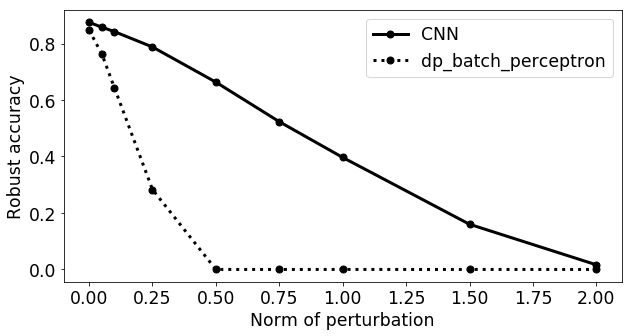}
\subcaption{$\eps = 2$}
\endminipage
\caption{Robustness accuracy comparison between DP-SGD-trained Convolutional neural networks and DP Batch Perceptron halfspace classifiers with Gaussian kernel on MNIST (top row) and USPS (bottom row) datasets for a fixed privacy budget. In all three plots, $\delta=10^{-5}$ but $\eps$ varies from 0.5, 1, and 2.}
\label{fig:robustness_mnist_usps_kernel} 
\end{figure*}

As for noise addition, the robustness guarantee for such kernel classifiers is also more complicated than the non-kernel linear classifiers. In Section~\ref{subsec:kernel-margin} below, we provide a \emph{provable} robustness guarantee of the kernel classifiers. Using this provable guarantee, the empirical robustness accuracy is shown in the last column of Figure \ref{fig:mnist_usps_kernel} and its comparison to DP-SGD-trained CNNs is shown in Figure~\ref{fig:robustness_mnist_usps_kernel}. Even though the kernel classifiers start off with similar accuracy (at $\gamma = 0$), it quickly drops and becomes worse than CNNs. We remark here that, in addition to the nature of the kernel, this may also be exacerbated by the fact that the provable robust guarantee for the kernel classifiers is \emph{not} tight (unlike the non-kernel case).

\subsubsection{Robustness Guarantee for Multi-Class Perceptron with Kernel}
\label{subsec:kernel-margin}

To compute the robust error for the kernel classifiers, we will use the following result, which is a slightly simplified version of Theorem 2.1 from~\citep{HeinA17}.

\begin{lemma} \label{lem:margin-general}
Let $M$ be a classifier which, for each class $y \in \{1, \dots, k\}$, computes some function $f^y: \R^d \to \R$ and predicts the class $y^*$ that minimizes $f^{y^*}$. Then, for every example $(\bx, y)$ and any $\Delta \in \R^d$ such that
\begin{align*}
\|\Delta\| \leq \min_{y' \ne y} \frac{f^y(\bx) - f^{y'}(\bx)}{\sup_{\bx' \in \R^d} \|\nabla f^y(\bx') - \nabla f^{y'}(\bx')\|},
\end{align*}
the classifier $M$ predicts $y$ on $\bx + \Delta$.
\end{lemma}

Note that this lemma is tight for the non-kernel case, leading to the margin formula $\gamma < \min_{y' \ne y} \frac{\left<\bw^{(y)}, \bx\right> - \left<\bw^{(y')}, \bx\right>}{\|\bw^{(y)} - \bw^{(y')}\|}$ that we used earlier.

Our kernel classifier is of the form in Lemma~\ref{lem:margin-general} with $f^y(\bx') := \left<\bw^{(y)}, \phi_{\rho_1, \dots, \rho_{\hd}}(\bx')\right>$. To apply the lemma, we first compute $\nabla f$:
\begin{align*}
\nabla f^y(\bx') = \frac{1}{\sqrt{\hd}} \cdot \sum_{i=1}^{\hd} \left(-w^{(y)}_i \cdot \sin(\left<\rho_i, \bx'\right>) + w^{(y)}_{\hd + i} \cdot \cos(\left<\rho_i, \bx'\right>)\right) \cdot \rho_i.
\end{align*}
As a result, for two classes $y, y'$, we have
\begin{align*}
\nabla f^y(\bx') - \nabla f^{y'}(\bx') = \frac{1}{\sqrt{\hd}} \cdot \sum_{i=1}^{\hd} \left((w^{(y')}_i - w^{(y)}_i) \cdot \sin(\left<\rho_i, \bx'\right>) + (w^{(y)}_{\hd + i} - w^{(y')}_{\hd + i}) \cdot \cos(\left<\rho_i, \bx'\right>)\right) \cdot \rho_i.
\end{align*}

In the following, we will give an upper bound on $\|\nabla f^y - \nabla f^{y'}\|$. Let $\Pi \in \R^{d \times \hd}$ resulting from concatenating $\rho_1, \dots, \rho_{\hd}$, and let $\bp \in \R^{\hd}$ denote the vector for which $p_i = \frac{1}{\sqrt{\hd}}(w^{y'}_i - w^y_i) \cdot \sin(\left<\rho_i, \bx'\right>) + (w^y_{\hd + i} - w^{y'}_{\hd + i}) \cdot \cos(\left<\rho_i, \bx'\right>)$. First, notice that
\begin{align*}
\nabla f^y - \nabla f^{y'} = \Pi\bp.
\end{align*}
Now, we may bound $\|\bp\|$ by
\begin{align*}
\|\bp\| &= \frac{1}{\sqrt{\hd}} \cdot \sqrt{\sum_{i=1}^{\hd} \left((w^{(y')}_i - w^{(y)}_i) \cdot \sin(\left<\rho_i, \bx'\right>) + (w^{(y)}_{\hd + i} - w^{(y')}_{\hd + i}) \cdot \cos(\left<\rho_i, \bx'\right>)\right)^2} \\
(\text{Cauchy–Schwarz inequality}) &\leq \frac{1}{\sqrt{\hd}} \cdot \sqrt{\sum_{i=1}^{\hd} \left((w^{(y')}_i - w^{(y)}_i)^2 + (w^{(y)}_{\hd + i} - w^{(y')}_{\hd + i})^2\right)\left(\sin(\left<\rho_i, \bx'\right>)^2 + \cos(\left<\rho_i, \bx'\right>)^2\right)} \\
&= \frac{1}{\sqrt{\hd}} \cdot \sqrt{\sum_{i=1}^{\hd} \left((w^{(y')}_i - w^{(y)}_i)^2 + (w^{(y)}_{\hd + i} - w^{(y')}_{\hd + i})^2\right)} \\
&= \frac{1}{\sqrt{\hd}} \cdot \|\bw^{(y)} - \bw^{(y')}\|
\end{align*}
As a result, we have
\begin{align*}
\|\nabla f^y(\bx') - \nabla f^{y'}(\bx')\| = \|\Pi\bp\| \leq \sigma_{\max}(\Pi) \cdot \frac{1}{\sqrt{\hd}} \cdot \|\bw^{(y)} - \bw^{(y')}\|,
\end{align*}
where $\sigma_{\max}(\Pi)$ denote the largest singular value of $\Pi$ (i.e. the operator norm of $\Pi$ with respect to $L_2$ norm).

Plugging this back into Lemma~\ref{lem:margin-general}, we can conclude that each example $(\bx, y)$ remaining correctly classifies up to perturbation norm of
\begin{align*}
\frac{\sqrt{\hd}}{\sigma_{\max}(\Pi)} \cdot \min_{y' \ne y} \frac{f^y(\bx) - f^{y'}(\bx)}{\|\bw^{(y)} - \bw^{(y')}\|}.
\end{align*}

\subsection{Comparison with Support Vector Machines (SVM)}

Previous work has introduced different approaches to preserving DP for SVM \citep{rubinstein2009learning}, or convex optimization algorithms in general \citep{feldman2020private, bassily2019private}. Our implementation of DP SVM uses DP SGD~\citep{abadi2016deep} with the standard hinge loss and $L_2$ weight regularization. The regularization strength is chosen from 1, 0.1, 0.01, 0.001, 0.0001, 0.00001, the learning rate from 1, 0.1, 0.01, 0.001, 0.0001. After summing gradients from each batch of data, we add appropriately calibrated Gaussian noise (again, based on Renyi DP) to the weights update.

Figures \ref{fig:mnist_usps_with_svm} and \ref{fig:mnist_usps_kernel_with_svm} compare performance of DP Batch Perceptron and DP SVM with and without kernel respectively. For experiments with varying $\eps$ (first column in both figures), we observe that DP Batch Perceptron outperforms DP SVM in most instances and achieves competitive accuracy on both datasets. For different $\delta$ values (second column) while keeping $\eps$ fixed at 1.0, the trend still holds to a large extent and both algorithms yield very similar test accuracy. The last column compares the robust accuracy of models trained via DP Batch Perceptron and DP SVM at $\eps = 1, 2$. In the case where no kernel is involved, the former yields better results on MNIST dataset but performs worse on USPS dataset. The opposite trend is observed when Gaussian kernel is included. 


\begin{figure}[ht]
\minipage{0.38\textwidth}
\includegraphics[width=\linewidth]{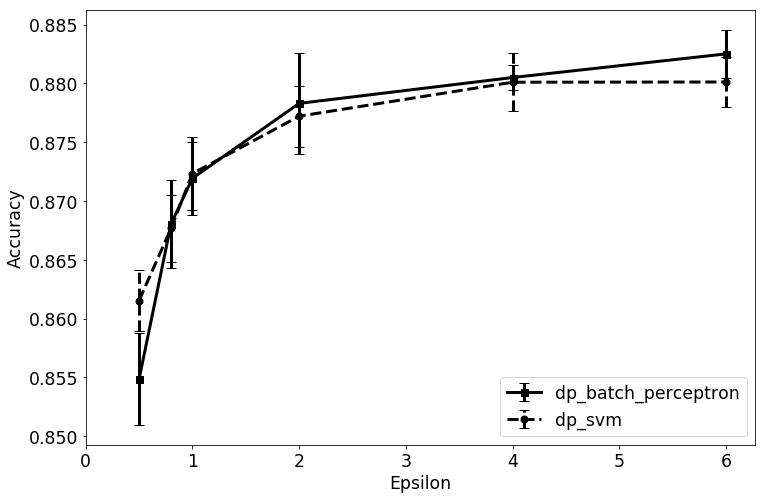}
\endminipage\hspace{-30pt}\hfill
\minipage{0.25\textwidth}
\includegraphics[width=1.1\linewidth]{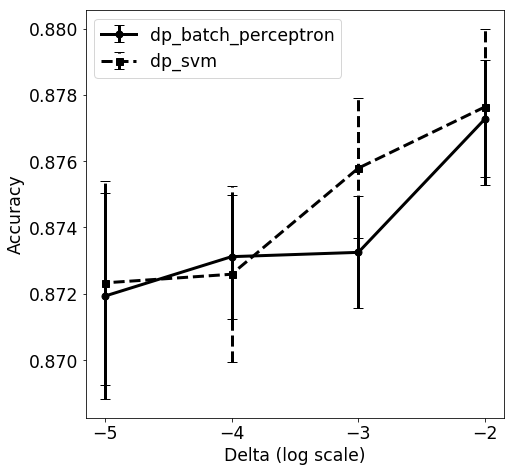}
\endminipage\hspace{-15pt}\hfill
\minipage{0.33\textwidth}
\includegraphics[width=\linewidth]{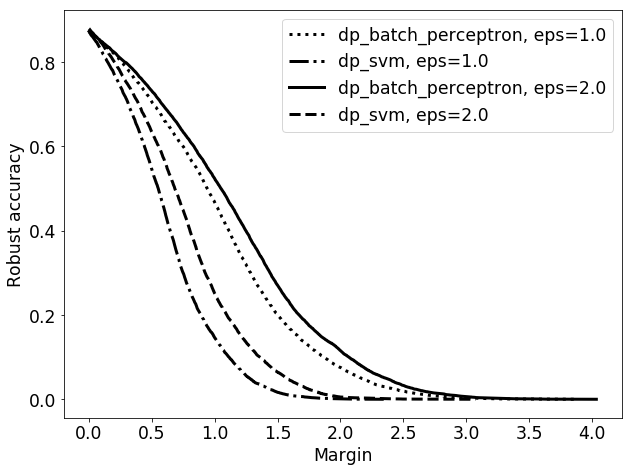}
\endminipage

\minipage{0.38\textwidth}
\includegraphics[width=\linewidth]{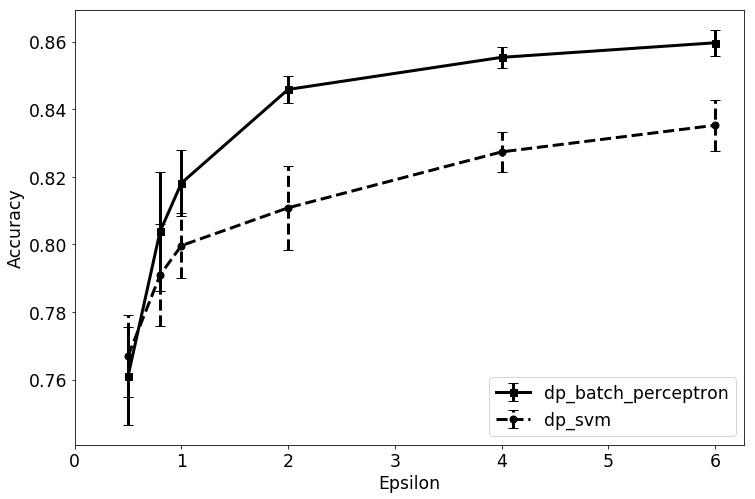}
\subcaption{Accuracy as $\eps$ varies}
\endminipage\hspace{-30pt}\hfill
\minipage{0.25\textwidth}
\includegraphics[width=1.1\linewidth]{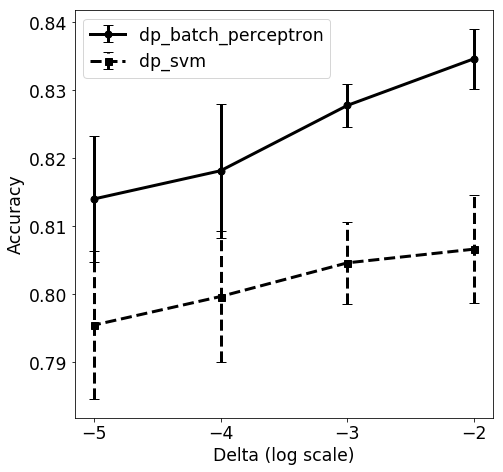}
\subcaption{Accuracy as $\delta$ varies}
\endminipage\hspace{-15pt}\hfill
\minipage{0.33\textwidth}
\includegraphics[width=\linewidth]{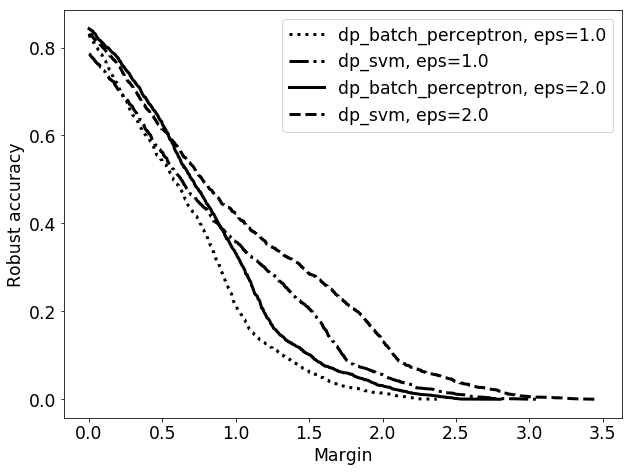}
\subcaption{Robust accuracy as $\eps$ varies}
\endminipage
\caption{Comparison of performance of DP Batch Perceptron vs DP SVM halfspace classifiers on the MNIST (top row) and USPS (bottom row) datasets, when no kernel is involved in the learning process.}
\label{fig:mnist_usps_with_svm}
\end{figure}

\begin{figure}[ht]
\minipage{0.38\textwidth}
\includegraphics[width=\linewidth]{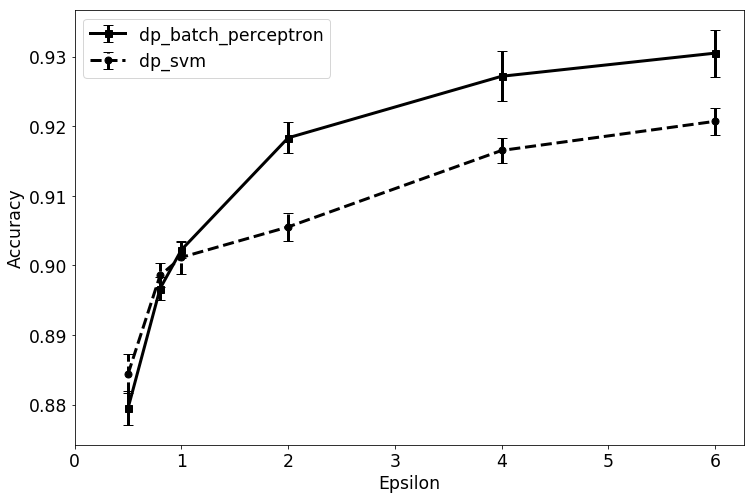}
\endminipage\hspace{-30pt}\hfill
\minipage{0.25\textwidth}
\includegraphics[width=1.1\linewidth]{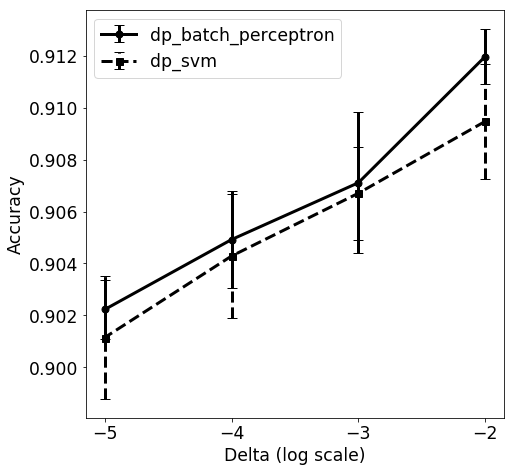}
\endminipage\hspace{-15pt}\hfill
\minipage{0.33\textwidth}
\includegraphics[width=\linewidth]{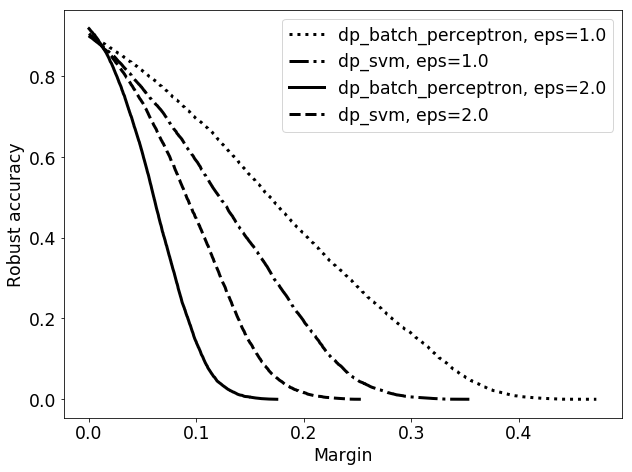}
\endminipage

\minipage{0.38\textwidth}
\includegraphics[width=\linewidth]{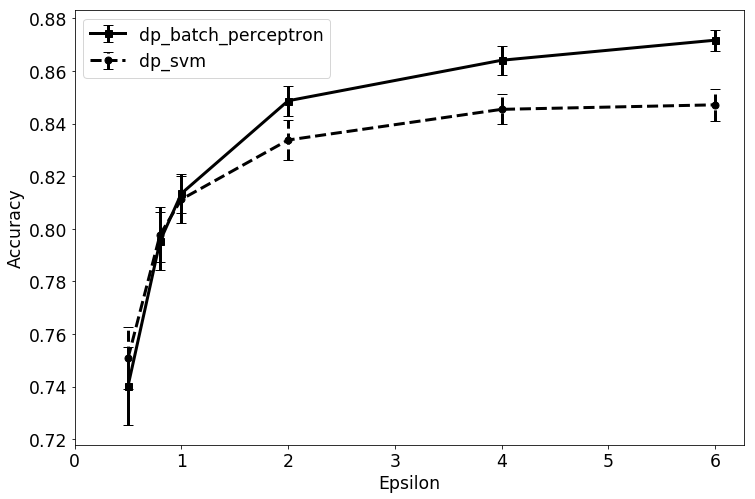}
\subcaption{Accuracy as $\eps$ varies}
\endminipage\hspace{-30pt}\hfill
\minipage{0.25\textwidth}
\includegraphics[width=1.1\linewidth]{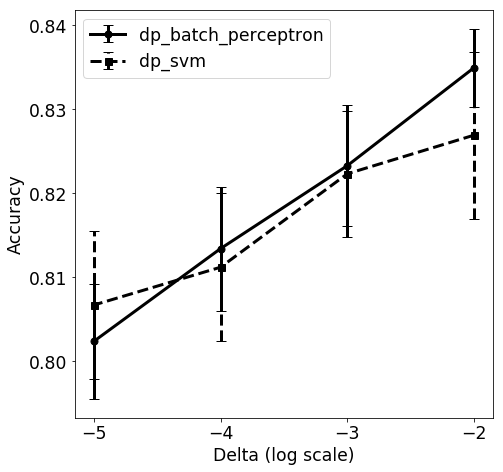}
\subcaption{Accuracy as $\delta$ varies}
\endminipage\hspace{-15pt}\hfill
\minipage{0.33\textwidth}
\includegraphics[width=\linewidth]{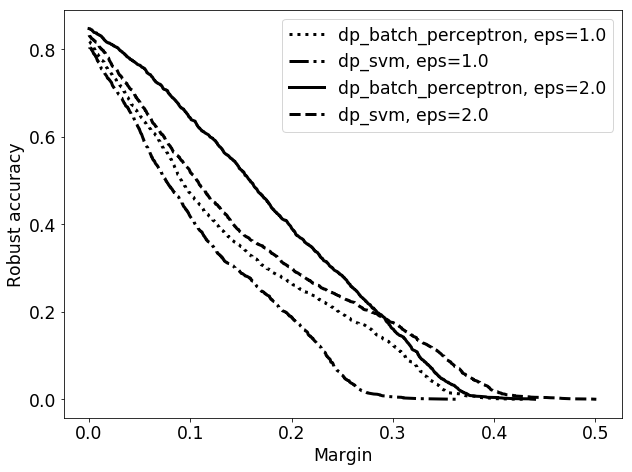}
\subcaption{Robust accuracy as $\eps$ varies}
\endminipage
\caption{Comparison of performance of DP Batch Perceptron vs DP SVM halfspace classifiers on the MNIST (top row) and USPS (bottom row) datasets, with Gaussian kernel.}
\label{fig:mnist_usps_kernel_with_svm}
\end{figure}

\end{document}